\documentclass{article}
\usepackage{ifthen}
\usepackage{mdwlist}
\usepackage{amsmath,amssymb,amsfonts,amsthm}
\usepackage{bm}
\usepackage{mathtools}
\usepackage{url}
\usepackage{hyperref}
\usepackage{enumitem}
\usepackage{graphicx}
\usepackage{xspace}
\usepackage{verbatim}
\usepackage[margin=1in]{geometry}
\usepackage{color}
\usepackage{latexsym}
\usepackage{epsfig}
\usepackage{bm}
\usepackage[noend]{algpseudocode}
\usepackage{algorithm}
\usepackage{etoolbox}
\usepackage{tikz}
\usepackage{pgfplots}
\pgfplotsset{compat=1.7}
\usepackage{hyperref}
\usepackage{subfigure}

\def\eps{\epsilon}
\def\ve{\varepsilon}

\def\to{\rightarrow}

\newcommand{\prob}[2][]{\text{\bf Pr}\ifthenelse{\not\equal{}{#1}}{_{#1}}{}\!\left[#2\right]}
\newcommand{\expect}[2][]{\text{\bf E}\ifthenelse{\not\equal{}{#1}}{_{#1}}{}\!\left[#2\right]}

\newcommand{\wh}[1]{{\widehat{#1}}}

\newtheorem{theorem}{Theorem}[section]
\newtheorem{lemma}[theorem]{Lemma}

\newtheorem{proposition}[theorem]{Proposition}
\newtheorem{corollary}[theorem]{Corollary}
\newtheorem{claim}[theorem]{Claim}
\newtheorem{definition}[theorem]{Definition}

\newtheorem{fact}[theorem]{Fact}

\newcommand{\ignore}[1]{}

\newcommand{\bg}[1]{\medskip\noindent{\bf #1}}

\definecolor{Red}{rgb}{1,0,0}

\newcommand{\oldbound}[1]{{}}

\newcommand{\Sigmahat}{\widehat{\Sigma}}

\newcommand{\Cov}{\mbox{Cov}}

\renewcommand{\epsilon}{\varepsilon}
\newcommand{\tr}{\mathrm{tr}}

\DeclareMathOperator{\R}{\mathbb{R}}

\DeclareMathOperator*{\Var}{Var}

\DeclareMathOperator*{\E}{\mathbb{E}}

\DeclareMathOperator{\poly}{poly}

\DeclareMathOperator{\rank}{rank}

\DeclareMathOperator{\normal}{\mathcal{N}}

\renewcommand{\[ }{\begin{eqnarray*}}
\renewcommand{\]}{\end{eqnarray*}}

\definecolor{darkpastelred}{rgb}{0.76, 0.23, 0.13}

\algnewcommand\INPUT{\item[{\textbf {input:}}]}
\algnewcommand\OUTPUT{\item[{\textbf{output:}}]}

\def\colorful{1}

\ifnum\colorful=0
\newcommand{\new}[1]{{\color{red} #1}}

\else
\newcommand{\new}[1]{{#1}}

\fi

\def\E{\mathbb{E}}

\newcommand{\eqdef}{\stackrel{{\mathrm {\footnotesize def}}}{=}}

\newcommand{\thres}{\mathrm{Thres}}
\newcommand{\tail}{\mathrm{Tail}}

\usepackage{ifthen}
\usepackage{mdwlist}
\usepackage{amsmath,amssymb,amsfonts,amsthm}
\usepackage{bm}
\usepackage{hyperref}
\usepackage{enumitem}
\usepackage{graphicx}
\usepackage{xspace}
\usepackage{verbatim}
\usepackage{algorithm}
\usepackage[margin=1in]{geometry}
\usepackage{color}
\usepackage{thm-restate}
\usepackage{latexsym}
\usepackage{epsfig}
\usepackage[symbol]{footmisc}

\title{Being Robust (in High Dimensions) Can Be Practical\footnote{A version of this paper appeared in ICML 2017~\cite{DiakonikolasKKLMS17}.}}

\author {
Ilias Diakonikolas\thanks{Supported by NSF  CAREER Award CCF-1652862, a Sloan Research Fellowship, and a Google Faculty Research Award.} \\
CS, USC \\
\tt{diakonik@usc.edu}
\and
Gautam Kamath\thanks{Supported by NSF CCF-1551875, CCF-1617730, CCF-1650733, and ONR N00014-12-1-0999.} \\
EECS \& CSAIL, MIT \\
\tt{g@csail.mit.edu}
\and
Daniel M. Kane\thanks{Supported by NSF  CAREER Award CCF-1553288 and a Sloan Research Fellowship.} \\
CSE \& Math, UCSD \\
\tt{dakane@cs.ucsd.edu}
\and
Jerry Li \thanks{Supported by NSF CAREER Award CCF-1453261, a Google Faculty Research Award, and an NSF Fellowship.}\\
EECS \& CSAIL, MIT \\
\tt{jerryzli@mit.edu}
\and
Ankur Moitra\thanks{Supported by NSF CAREER Award CCF-1453261, a grant from the MIT NEC Corporation, and a Google Faculty Research Award.} \\
Math \& CSAIL, MIT \\
\tt{moitra@mit.edu}
\and
Alistair Stewart\thanks{Research supported by a USC startup grant.}\\
CS, USC \\
\tt{alistais@usc.edu}
}

\begin{document}
\maketitle \footnotetext{Authors are in alphabetical order.}
\footnotetext{Code of our implementation is available at \url{https://github.com/hoonose/robust-filter}.}

\begin{abstract} 
Robust estimation is much more challenging in high dimensions than it is in one dimension: 
Most techniques either lead to intractable optimization problems or estimators 
that can tolerate only a tiny fraction of errors. Recent work in theoretical computer science has shown that, 
in appropriate distributional models, it is possible to robustly estimate the mean and covariance with polynomial time 
algorithms that can tolerate a constant fraction of corruptions, independent of the dimension. 
However, the sample and time complexity of these algorithms is prohibitively large for high-dimensional applications. 
In this work, we address both of these issues by establishing sample complexity bounds that are optimal, up to logarithmic factors, 
as well as giving various refinements that allow the algorithms to tolerate a much larger fraction of corruptions. 
Finally, we show on both synthetic and real data that our algorithms have state-of-the-art performance and suddenly make 
high-dimensional robust estimation a realistic possibility.  
\end{abstract} 

\section{Introduction}


Robust statistics was founded in the seminal works of \cite{tukey1960} and \cite{huber1964}.
The overarching motto is that any model (especially a parametric one) is only approximately valid, 
and that any estimator designed for a particular distribution that is to be used in practice 
must also be stable in the presence of model misspecification. The standard setup is to assume 
that the samples we are given come from a nice distribution, but that an adversary has the 
power to arbitrarily corrupt a constant fraction of the observed data. After several decades of work, 
the robust statistics community has discovered a myriad of estimators that are provably robust. 
An important feature of this line of work is that it can tolerate a constant fraction of 
corruptions {\em independent of the dimension} 
and that there are estimators for both the location (e.g., the mean) and scale (e.g., the covariance). 
See \cite{Huber09} and \cite{HampelEtalBook86} for further background. 

It turns out that there are vast gaps in our understanding of robustness, 
when computational considerations are taken into account. In one dimension, robustness and computational efficiency 
are in perfect harmony. The empirical mean and empirical variance are not robust, because a single corruption 
can arbitrarily bias these estimates, but alternatives such as the median and the interquartile range 
are straightforward to compute and are provably robust.

But in high dimensions, there is a striking tension between robustness and computational efficiency. 
Let us consider estimators for location. The Tukey median \cite{tukey1960} is a natural generalization 
of the one-dimensional median to high-dimensions. It is known that it behaves well (i.e., it needs few samples) 
when estimating the mean for various symmetric distributions \cite{Donoho92, CGR15b}. However, it is hard 
to compute in general \cite{JP:78, AmaldiKann:95} and the many heuristics for computing it degrade 
badly in the quality of their approximation as the dimension scales \cite{Clarkson93, Chan04, MillerS10}. 
The same issues plague estimators for scale. The minimum volume ellipsoid~\cite{Rous85} is a 
natural generalization of the one-dimensional interquartile range and is provably robust in high-dimensions, 
but is also hard to compute. And once again, heuristics for computing it \cite{WICS09, Rousseeuw98} work poorly in high dimensions.

The fact that robustness in high dimensions seems to come at such a steep price has long been a point of 
consternation within robust statistics. In a 1997 retrospective on the development of robust statistics~\cite{Huber97}, Huber laments:

\begin{quote}
``It is one thing to design a theoretical algorithm whose purpose is to prove [large fractions of corruptions can be tolerated] and quite another thing 
to design a practical version that can be used not merely on small, 
but also on medium sized regression problems, with a $2000$ by $50$ matrix or so. 
This last requirement would seem to exclude all of the recently proposed [techniques]."
\end{quote}

\noindent The goal of this paper is to answer Huber's call to action and design estimators 
for both the mean and covariance that are highly practical, provably robust, and work in high-dimensions. 
Such estimators make the promise of robust statistics \--- estimators that work in high-dimensions 
and guarantee that their output has not been heavily biased by some small set of noisy samples \--- much closer to a reality. 

First, we make some remarks to dispel some common misconceptions. There has been a considerable amount 
of recent work on robust principal component analysis, much of it making use of semidefinite programming. 
Some of these works can tolerate a constant fraction of corruptions \cite{CLMW11}, however require that the locations 
of the corruptions are evenly spread throughout the dataset 
so that no individual sample is entirely corrupted. In contrast, the usual models in robust statistics 
are quite rigid in what they require and they do this for good reason. A common scenario that is used 
to motivate robust statistical methods is if two studies are mixed together, and one subpopulation does not fit the model. 
Then one wants estimators that work without assuming anything at all about these outliers. 

There have also been semidefinite programming methods proposed for robust principal component analysis with outliers \cite{xu2010robust}. 
These methods assume that the uncorrupted matrix is rank $r$ and that the fraction of outliers is at most $1/r$, 
which again degrades badly as the rank of the matrix increases. Moreover, any method that uses semidefinite programming 
will have difficulty scaling to the sizes of the problems we consider here. For sake of comparison \--- even with state-of-the-art interior point methods \--- 
it is not currently feasible to solve the types of semidefinite programs that have been proposed when the matrices have dimension larger than a hundred. 

\subsection{Robustness in a Generative Model}

Recent works in theoretical computer science have sought to circumvent 
the usual difficulties of designing efficient and robust algorithms by instead working in a generative model. 
The starting point for our paper is the work of \cite{DKKLMS} who gave an efficient algorithm 
for the problem of {\em agnostically learning a Gaussian}:

\begin{quote}
Given a polynomial number of samples from a high-dimensional Gaussian $\mathcal{N}(\mu, \Sigma)$, 
where an adversary has arbitrarily corrupted an $\eps$-fraction, find a set of parameters $\mathcal{N}'(\wh{\mu}, \wh{\Sigma})$ 
that satisfy $d_{TV}(\mathcal{N}, \mathcal{N}') \leq \widetilde{O}(\eps)$\footnote{We use the notation $\tilde O(\cdot)$ to hide factors which are polylogarithmic in the argument -- in particular, we note that this bound does not depend on the dimension.}. 
\end{quote}

Total variation distance is the natural metric to use to measure closeness of the parameters, 
since a $(1-\epsilon)$-fraction of the observed samples came from a Gaussian. 
\cite{DKKLMS} gave an algorithm for the above problem (note that the guarantees are dimension independent), 
whose running time and sample complexity are polynomial in the dimension $d$ and $1/\eps$. 
\cite{LaiRV16} independently gave an algorithm for the unknown mean case 
that achieves $d_{TV}(\mathcal{N}, \mathcal{N}') \leq \widetilde{O}(\epsilon \sqrt{\log d})$, 
and in the unknown covariance case achieves guarantees in a weaker metric that is not affine invariant. 
A crucial feature is that both algorithms work even when the moments of the underlying distribution 
satisfy certain conditions, and thus are not necessarily brittle to the modeling assumption 
that the inliers come from a Gaussian distribution. 

A more conceptual way to view such work is as a proof-of-concept that the Tukey median 
and minimum volume ellipsoid can be computed efficiently {\em in a natural family of distributional models}. 
This follows because not only would these be good estimates for the mean and covariance in the above model, 
but in fact any estimates that are good must also be close to them. 
Thus, these works fit into the emerging research direction of circumventing worst-case lower bounds 
by going {\em beyond worst-case analysis}. 

Since the dissemination of the aforementioned works \cite{DKKLMS, LaiRV16}, 
there has been a flurry of research activity on computationally efficient robust estimation 
in a variety of high-dimensional settings~\cite{DiakonikolasKS16b, DiakonikolasKS16c, CharikarSV16, DiakonikolasKKLMS17, Li17, DBS17, BalakrishnanDLS17, SteinhardtCV17, DiakonikolasKKLMS18}, 
including studying graphical distributional models~\cite{DiakonikolasKS16b}, 
understanding the computation-robustness tradeoff for statistical query algorithms~\cite{DiakonikolasKS16c}, tolerating much more noise 
by allowing the algorithm to output a list of candidate hypotheses~\cite{CharikarSV16}, and developing robust algorithms 
under sparsity assumptions~\cite{Li17, DBS17, BalakrishnanDLS17}, where the number of samples is sublinear in the dimension.

\subsection{Our Results} \label{ssec:our-results}

Our goal in this work is to show that high-dimensional robust estimation can be highly practical. 
However, there are two major obstacles to achieving this. First, the sample complexity and running time 
of the algorithms in \cite{DKKLMS} is prohibitively large for high-dimensional applications. 
We just would not be able to store as many samples as we would need, in order to compute accurate estimates, 
in high-dimensional applications. 

Our first main contribution is to show essentially tight bounds on the sample complexity 
of the filtering based algorithm of \cite{DKKLMS}. Roughly speaking, we accomplish this with a new definition 
of the {\em good set} which plugs into the existing analysis in a straightforward manner 
and shows that it is possible to estimate the mean with $\widetilde{O}(d/\epsilon^2)$ samples (when the covariance is known) 
and the covariance with $\widetilde{O}(d^2/\epsilon^2)$ samples. 
Both of these bounds are information-theoretically optimal, up to logarithmic factors. 

Our second main contribution is to vastly improve the fraction of adversarial corruptions that can be tolerated in applications. 
The fraction of errors that the algorithms of \cite{DKKLMS} can tolerate is indeed a constant that is independent of the dimension, 
but it is very small both in theory and in practice. This is due to the fact that many of the steps in the algorithm are overly conservative. 
In fact, we found that a naive implementation of the algorithm did not remove {\em any} outliers in many realistic scenarios. 
We combat this by giving new ways to empirically tune the threshold for where to remove points from the sample set. 
These optimizations dramatically improve the empirical performance. 

Finally, we show that the same bounds on the error guarantee continue to work even when the underlying distribution is sub-Gaussian. 
This theoretically confirms that the robustness guarantees of such algorithms are in fact not overly brittle to the distributional assumptions. 
In fact, the filtering algorithm of \cite{DKKLMS} is easily shown to be robust under much weaker distributional assumptions, while retaining 
near-optimal sample and error guarantees. As an example, we show that it yields a near sample-optimal efficient estimator
for robustly estimating the mean of a distribution, under the assumption that its covariance is bounded. Even in this regime,
the filtering algorithm guarantees optimal error, up to a constant factor.
Furthermore we empirically corroborate this finding by showing that the algorithm works well on real world data, as we describe below. 

Now we come to the task of testing out our algorithms. To the best of our knowledge, there have been 
no experimental evaluations of the performance of the myriad of approaches to robust estimation. 
It remains mostly a mystery which ones perform well in high-dimensions, and which do not. 
To test out our algorithms, we design a synthetic experiment where a $(1-\epsilon)$-fraction of the samples 
come from a Gaussian and the rest are noise and sampled from another distribution (in many cases, Bernoulli). 
This gives us a baseline to compare how well various algorithms recover $\mu$ and $\Sigma$, 
and how their performance degrades based on the dimension. Our plots show a predictable and yet striking phenomenon: 
All earlier approaches have error rates that scale polynomially with the dimension and ours is a constant 
that is almost indistinguishable from the error that comes from sample noise alone. 
Moreover, our algorithms are able to scale to hundreds of dimensions.

But are algorithms for agnostically learning a Gaussian unduly sensitive to the distributional assumptions they make? 
We are able to give an intriguing visual demonstration of our techniques on real data. 
The famous study of \cite{novembre2008genes} showed that performing principal component analysis 
on a matrix of genetic data recovers a map of Europe. More precisely, the top two singular vectors 
define a projection into the plane and when the groups of individuals are color-coded with where they are from, 
we recover familiar country boundaries that corresponds to the map of Europe. The conclusion from their study 
was that {\em genes mirror geography}. Given that one of the most important applications of robust estimation 
ought to be in exploratory data analysis, we ask: To what extent can we recover the map of Europe in the presence of noise? 
We show that when a small number of corrupted samples are added to the dataset, the picture becomes entirely distorted 
(and this continues to hold even for many other methods that have been proposed). In contrast, when we run our algorithm, 
we are able to once again recover the map of Europe. Thus, even when some fraction of the data has been corrupted 
(e.g., medical studies were pooled together even though the subpopulations studied were different), 
it is still possible to perform principal component analysis 
and recover qualitatively similar conclusions as if there were no noise at all!

\section{Formal Framework} \label{sec:prelims}
 
\noindent {\bf Notation.} For a vector $v$, we will let $\| v \|_2$ denote its Euclidean norm.
If $M$ is a matrix, we will let $\| M \|_2$ denote its spectral norm and $\| M \|_F$ denote its Frobenius norm.
We will write $X \in_u S$ to denote that $X$ is drawn from the empirical distribution defined by $S$.

\medskip
 
\noindent {\bf Robust Estimation.} 
We consider the following powerful model of robust estimation
that generalizes many other existing models, including Huber's contamination model:
\begin{definition} \label{def:adv}
Given $\eps > 0$ and a distribution family $\mathcal{D}$, 
the \emph{adversary} operates as follows: The algorithm specifies some number of samples $m$.
The adversary generates $m$ samples $X_1, X_2, \ldots, X_m$ from some (unknown) $D \in \mathcal{D}$.
It then draws $m'$ from an appropriate distribution.
This distribution is allowed to depend on $X_1, X_2, \ldots, X_m$,
but when marginalized over the $m$ samples satisfies $m' \sim \mbox{Bin} (\eps, m)$.
The adversary is allowed to inspect the samples, removes $m'$ of them, 
and replaces them with arbitrary points. The set of $m$ points is then given to the algorithm.
\end{definition}

In summary, the adversary is allowed to inspect the samples before corrupting them,
both by adding corrupted points and deleting uncorrupted points. In contrast, in Huber's model
the adversary is oblivious to the samples and is only allowed to add corrupted points.

We remark that there are no computational restrictions on the adversary.
The goal is to return the parameters of a distribution $\widehat{D}$ in $\mathcal{D}$
that are close to the true parameters in an appropriate metric.
For the case of the mean, our metric will be the Euclidean distance.
For the covariance, we will use the Mahalanobis distance, i.e., $\| \Sigma^{-1/2} \wh{\Sigma} \Sigma^{-1/2} -I \|_F$.
This is a strong affine invariant distance that implies corresponding bounds in total variation distance.

We will use the following terminology:
\begin{definition}
We say that a set of samples is $\eps$-corrupted 
if it is generated by the process described in Definition~\ref{def:adv}.
\end{definition}

\section{Nearly Sample-Optimal Efficient Robust Learning} \label{sec:optimal-sample}
In this section, we present near sample-optimal efficient robust estimators for the mean and the covariance
of high-dimensional distributions under various structural assumptions of varying strength.
Our estimators rely on the {\em filtering technique} introduced in~\cite{DKKLMS}. 

We note that \cite{DKKLMS} gave two algorithmic techniques:
the first one was a spectral technique to iteratively remove outliers from the dataset (filtering), 
and the second one was a soft-outlier removal method relying on convex programming. 
The filtering technique seemed amenable to practical implementation (as it only uses simple eigenvalue computations),
but the corresponding sample complexity bounds given in \cite{DKKLMS} 
are polynomially worse than the information-theoretic minimum. 
On the other hand, the convex programming technique of \cite{DKKLMS}  achieved better
sample complexity bounds (e.g., near sample-optimal for robust mean estimation), 
but relied on the ellipsoid method, which seemed to preclude a practically efficient implementation. 

In this work, we achieve the best of both worlds: 
we provide a more careful analysis of the filter technique that yields sample-optimal bounds (up to logarithmic factors) 
for  both the mean and the covariance. Moreover, we show that the filtering technique easily extends 
to much weaker distributional assumptions (e.g., under bounded second moments). 
Roughly speaking, the filtering technique follows a general iterative recipe: 
(1) via spectral methods, find some univariate test which is violated by the corrupted points, 
(2) find some concrete tail bound violated by the corrupted set of points, 
and (3) throw away all points which violate this tail bound.

\medskip

We start with sub-gaussian distributions. Recall that
if $P$ is sub-gaussian on $\R^d$ with mean vector $\mu$ and parameter $\nu>0$, 
then for any unit vector $v \in \R^d$ we have that 
$\Pr_{X \sim P}\left[|v \cdot (X-\mu)| \geq t \right] \leq \exp(-t^2/2\nu)$.


\begin{theorem} \label{thm:filter-gaussian-mean}
Let $G$ be a sub-gaussian distribution on $\R^d$ with parameter $\nu=\Theta(1)$, 
mean $\mu^G$, covariance matrix $I$, and $\eps > 0$.
Let $S$ be an $\eps$-corrupted set of samples from $G$ of size
\new{$\Omega((d/\eps^2) \poly\log(d/\eps))$}. 
There exists an efficient algorithm that, on input $S$ and $\eps>0$, returns a mean vector $\wh{\mu}$
so that with probability at least $9/10$ we have $\|\wh{\mu}-\mu^{G}\|_2 = O(\eps\sqrt{\log(1/\eps)}).$
\end{theorem}

\cite{DKKLMS}  gave algorithms for robustly estimating the mean of a Gaussian distribution with known 
covariance and for robustly estimating the mean of a binary product distribution. 
The main motivation for considering these specific distribution families
is that robustly estimating the mean within Euclidean distance immediately implies total variation distance bounds
for these families. The above theorem establishes that these guarantees hold in 
a more general setting with near sample-optimal bounds.
Under a bounded second moment assumption, we show:

\begin{theorem} \label{thm:second-moment}
Let $P$ be a distribution on $\R^d$ with unknown mean vector $\mu^{P}$ and unknown covariance matrix 
$\Sigma_P \preceq \sigma^2 I$. Let $S$ be an $\eps$-corrupted set of samples 
from $P$ of size $\Theta((d/\eps) \log d)$. 
There exists an efficient algorithm that, on input $S$ and $\eps>0$, with probability $9/10$ 
outputs $\wh{\mu}$ with $\|\wh{\mu} - \mu^{P}\|_2 \leq O(\sqrt{\eps} \sigma)$.
\end{theorem}

A similar result on mean estimation under bounded second moments was concurrently shown in \cite{SteinhardtCV17}.
The sample size above is optimal, up to a logarithmic factor, and the error guarantee 
is easily seen to be the best possible up to a constant factor. 
The main difference between the filtering algorithm establishing the above theorem 
and the filtering algorithm for the sub-gaussian case is how we choose the threshold for the filter. 
Instead of looking for a violation of a concentration inequality, here we will choose a threshold {\em at random}.
In this case, randomly choosing a threshold weighted towards higher thresholds suffices to throw 
out more corrupted samples than uncorrupted samples {\em in expectation}. 
Although it is possible to reject many good samples this way, we show that
the algorithm still only rejects a total of $O(\eps)$ samples with high probability.

Finally, for robustly estimating the covariance of a Gaussian distribution, we have:

\begin{theorem}\label{unknownCovarianceTheorem}
Let $G \sim \normal(0, \Sigma)$ be a Gaussian in $d$ dimensions, and let $\eps>0$. 
Let $S$ be an $\eps$-corrupted set of samples from $G$ of size \new{$\Omega((d^2/\eps^2) \poly\log(d/\eps))$}.
There exists an efficient algorithm that, given $S$ and $\eps$, 
returns the parameters of a Gaussian distribution $G'  \sim \normal(0, \wh{\Sigma})$ 
so that with probability at least $9/10$, 
it holds $\|I  - \Sigma^{-1/2} \wh{\Sigma} \Sigma^{-1/2} \|_F  = O(\eps\log(1/\eps)).$
\end{theorem}


We now provide a high-level description of the main ingredient 
which yields these improved sample complexity bounds.
The initial analysis of \cite{DKKLMS} established sample complexity bounds which were sub-optimal by polynomial factors 
because it insisted that the set of good samples (i.e., before the corruption) 
satisfied very tight tail bounds. To some degree such bounds are necessary, 
as when we perform our filtering procedure, we need to ensure that not too 
many good samples are thrown away. However, the old analysis required that fairly
strong tail bounds hold {\em uniformly}. The idea for the improvement is as
follows: If the errors are sufficient to cause the variance of some
polynomial $p$ (linear in the unknown mean case or quadratic in the
unknown covariance case) to increase by more than $\eps$, it must be
the case that for some $T$, roughly an $\eps/T^2$ fraction of samples are
error points with $|p(x)| > T$. As long as we can ensure that less than
an $\eps/T^2$ fraction of our good sample points have $|p(x)| > T$, this
will suffice for our filtering procedure to work. For small values
of $T$, these are much weaker tail bounds than were needed previously 
and can be achieved with a smaller number of samples.
For large values of $T$, these tail bounds are comparable to those used
in previous work \cite{DKKLMS} , but in such cases we can take advantage of the
fact that $|p(G)| > T$ only with very small probability, again allowing
us to reduce the sample complexity. The details are deferred to Appendix~\ref{sec:app-theory}.

\section{Filtering}
We now describe the filtering technique more rigorously. We also describe some additional heuristics we found useful in practice.

\subsection{Robust Mean Estimation}
We first consider mean estimation.
The algorithms which achieve Theorems \ref{thm:filter-gaussian-mean} and \ref{thm:second-moment} both follow the general recipe in Algorithm \ref{alg:filter-Gaussian-template}.
We must specify three parameter functions:
\begin{itemize}
\item $\thres(\eps)$ is a threshold function---we terminate if the covariance has spectral norm bounded by $\thres(\eps)$.
\item $\tail (T, d, \eps, \delta, \tau)$ is an univariate tail bound, which would only be violated by a $\tau$ fraction of points if they were uncorrupted, but is violated by many more of the current set of points.
\item $\delta(\eps, s)$ is a slack function, which we require for technical reasons.
\end{itemize}
Given these objects, our filter is fairly easy to state: first, we compute  the empirical covariance.
Then, we check if the spectral norm of the empirical covariance exceeds $\thres(\eps)$.
If it does not, we output the empirical mean with the current set of data points.
Otherwise, we project onto the top eigenvector of the empirical covariance, and throw away all points which violate $\tail(T, d, \eps, \delta, \tau)$, for some choice of slack function $\delta$.

\begin{algorithm}
\begin{algorithmic}[1]
\State \textbf{Input:} An $\epsilon$-corrupted set of samples $S$, $\thres(\eps), \tail (T, d, \eps, \delta, \tau), \delta(\eps, s)$
\State Compute the sample mean $\mu^{S'}=\E_{X\in_u S'}[X]$ 
\State Compute the sample covariance matrix $\Sigma$ 
\State  Compute approximations for the largest absolute eigenvalue of $\Sigma$, $\lambda^{\ast} := \|\Sigma\|_2,$
and the associated unit eigenvector $v^{\ast}.$

\If{$\|\Sigma\|_2 \leq \thres(\eps)$}
 \State \textbf{return} $\mu^{S'}.$ 
\EndIf
\State Let $\delta = \delta (\eps, \| \Sigma \|_2)$.
\State Find $T>0$ such that
$$
\Pr_{X\in_u S'}\left[|v^{\ast} \cdot (X-\mu^{S'})|>T+\delta \right] > \tail(T, d, \eps, \delta, \tau).
$$
\State \Return $\{x\in S': |v^{\ast} \cdot (x-\mu^{S'}) | \leq T+\delta\}$.

\end{algorithmic}
\caption{Filter-based algorithm template for robust mean estimation}
\label{alg:filter-Gaussian-template}
\end{algorithm}

\paragraph{Sub-gaussian case} 
To concretely instantiate this algorithm for the subgaussian case, we take $\thres (\eps) = O(\eps \log 1 / \eps)$, $\delta(\ve, s) = 3 \sqrt{\eps (s - 1)}$, and
\[
\tail (T, d, \eps, \delta, \tau) = 8 \exp (-T^2 / 2 \nu) + 8 \frac{\eps}{T^2 \log (d \log (d / \eps \tau))} \; ,
\]
where $\nu$ is the subgaussian parameter.
See Section~\ref{sec:filter-subgaussian} for details.

\paragraph{Second moment case}
To concretely instantiate this algorithm for the second moment case, 
we take $\thres (\eps) = 9$, $\delta = 0$, and we take $\tail$ to be a \emph{random rescaling} 
of the largest deviation in the data set, in the direction $v^\ast$.
See Section~\ref{ssec:filter-2mom} for details.

\subsection{Robust Covariance Estimation}
Our algorithm for robust covariance follows the exact recipe outlined above, 
with one key difference---we check for deviations in the empirical \emph{fourth} moment tensor.
Intuitively, just as in the robust mean setting, we used degree-$2$ information to detect outliers for the mean 
(the degree-$1$ moment), here we use degree-$4$ information to detect outliers for the covariance 
(the degree-$2$ moment).

More concretely, this corresponds to finding a normalized degree-$2$ polynomial whose empirical variance is too large.
By then filtering along this polynomial, with an appropriate choice of $\thres (\eps), \delta (\eps, s),$ and $\tail$, we achieve the desired bounds.
See Section~\ref{ssec:cov} for the formal pseudocode and more details.

\subsection{Better Univariate Tests}
In the algorithms described above for robust mean estimation, after projecting onto one dimension, we center the points at the empirical mean along this direction.
This is theoretically sufficient, however, introduces additional constant factors since the empirical mean along this direction may be corrupted.
Instead, one can use a robust estimate for the mean in one direction.
Namely, it is well known that the median is a provably robust estimator for the mean for symmetric distributions~\cite{Huber09, HampelEtalBook86}, and under certain models it is in fact optimal in terms of its resilience to noise~\cite{DKW56,Massart90,Chen98,DK14,DiakonikolasKKLMS17}.
By centering the points at the median instead of the mean, we are able to achieve better error in practice.

\subsection{Adaptive Tail Bounding}
In our empirical evaluation, we found that it was important to find an appropriate choice of $\tail$, to achieve good error rates, especially for robust covariance estimation.
Concretely, in this setting, our tail bound is given by
\[
\tail (T, d, \eps, \delta, \tau) = C_1 \exp(-C_2 T) + \tail_2 (T, d, \eps, \delta, \tau) \;,
\]
for some function $\tail_2$, and constants $C_1, C_2$.
We found that for reasonable settings, the term that dominated was
always the first term on the RHS, and that $\tail_2$
is less significant. Thus, we focused on optimizing the first term.

We found that depending on the setting, it was useful to change the constant $C_2$.
In particular, in low dimensions, we could be more stringent, and enforce 
a stronger tail bound (which corresponds to a higher $C_2$), but in higher dimensions, we must be more lax with the tail bound.
To do this in a principled manner, we introduced a heuristic we call \emph{adaptive tail bounding}.
Our goal is to find a choice of $C_2$ which throws away roughly an $\eps$-fraction of points.
The heuristic is fairly simple: we start with some initial guess for $C_2$.
We then run our filter with this $C_2$.
If we throw away too many data points, we increase our $C_2$, and retry.
If we throw away too few, then we decrease our $C_2$ and retry.
Since increasing $C_2$ strictly decreases the number of points thrown away, and vice versa, 
we binary search over our choice of $C_2$ until we reach something close to our target accuracy.
In our current implementation, we stop when the fraction of points we throw away is between $\eps / 2$ and $3 \eps / 2$, 
or if we've binary searched for too long.
We found that this heuristic drastically improves our accuracy, 
and allows our algorithm to scale fairly smoothly from low to high dimension.


\pgfplotsset{ignore legend/.style={every axis legend/.code={\renewcommand\addlegendentry[2][]{}}}}

\pgfplotsset{resultplot/.style={%
  scale only axis,
  tick label style={font=\scriptsize},
  yticklabel style={
        /pgf/number format/fixed,
        /pgf/number format/precision=5
},
scaled y ticks=false,
  yticklabel shift={-2pt},
  enlarge x limits=false,
  xlabel shift=-3pt,
  ylabel shift=-5pt,
  label style={font=\scriptsize},
  title style={font=\scriptsize,yshift=-4pt},
  grid=major,
  grid style={dotted,gray,thin},
  legend cell align=left,
  legend transposed=false,
  legend columns=2,
  every axis plot/.append style={line width=1pt,mark size=2pt},
  cycle list={{black,mark=+},{blue,mark=o},{red,mark=asterisk},{olive,mark=diamond*},{brown,mark=otimes*}, {gray,mark=diamond*}}}}
\pgfplotsset{resultplot1/.style={%
  resultplot,
  height=\plotheight,
  width=\plotwidth}}

\newlength{\plotwidth}
\setlength{\plotwidth}{3cm}
\newlength{\plotheight}
\setlength{\plotheight}{2.6cm}
\newlength{\plotxspacing}
\setlength{\plotxspacing}{1.3cm}
\newlength{\plotyspacing}
\setlength{\plotyspacing}{1.3cm}

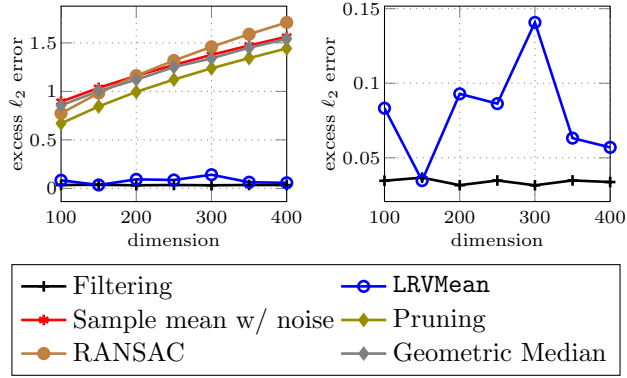
\begin{figure}[!t]
\centering

\pgfplotsset{mseplot/.style={%
  resultplot1,
  xlabel=$n$,
  ylabel=MSE}}
\pgfplotsset{mseratioplot/.style={%
  resultplot1,
  xlabel=$n$,
  ylabel=Relative MSE ratio}}
\pgfplotsset{errplot/.style={%
  resultplot1,
  xlabel=dimension,
  ylabel=excess $\ell_2$ error}}
\pgfplotsset{timeratioplot/.style={%
  resultplot1,
  xlabel=$n$,
  ylabel=Speed-up}}

\begin{tikzpicture}

\begin{axis}[errplot,name=mean, legend style={anchor=north west, xshift=-1.2 \plotwidth, yshift=-1.3\plotheight}]
\addplot table[x=d, y=err] {plot_data/excess-mean-filter-err.txt};
\addplot table[x=d, y=err] {plot_data/excess-mean-LRV-err.txt};
\addplot table[x=d, y=err] {plot_data/excess-mean-noisy-samp-err.txt};
\addplot table[x=d, y=err] {plot_data/excess-mean-pruned-err.txt};
\addplot table[x=d, y=err] {plot_data/excess-mean-median-err.txt};
\addplot table[x=d, y=err] {plot_data/excess-mean-ransac-err.txt};
\legend{Filtering ,\texttt{LRVMean}, Sample mean w/ noise, Pruning, RANSAC, Geometric Median}
\end{axis}

\begin{axis}[errplot,name=mean2, at=(mean.north east),anchor=north west, xshift=\plotxspacing,ignore legend]
\addplot table[x=d, y=err] {plot_data/excess-mean-filter-err.txt};
\addplot table[x=d, y=err] {plot_data/excess-mean-LRV-err.txt};\end{axis}
\end{tikzpicture}
\caption{Experiments with synthetic data for robust mean estimation: error is reported against dimension (lower is better).
The error is excess $\ell_2$ error over the sample mean without noise (the benchmark).
We plot performance of our algorithm, \texttt{LRVMean}, empirical mean with noise, pruning, RANSAC, and geometric median. On the left we report the errors achieved by all algorithms; however the latter four have much larger error than our algorithm or \texttt{LRVMean}. On the right, we restrict our attention to only our algorithm and \texttt{LRVMean}. Our algorithm has better error than all other algorithms.}
\label{fig:synthetic-mean}
\end{figure}

\begin{figure}[!t]
\centering

\pgfplotsset{mseplot/.style={%
  resultplot1,
  xlabel=$n$,
  ylabel=MSE}}
\pgfplotsset{mseratioplot/.style={%
  resultplot1,
  xlabel=$n$,
  ylabel=Relative MSE ratio}}
\pgfplotsset{errplot/.style={%
  resultplot1,
  xlabel=dimension,
  ylabel=excess Mahalanobis error}}
\pgfplotsset{timeratioplot/.style={%
  resultplot1,
  xlabel=$n$,
  ylabel=Speed-up}}

\begin{tikzpicture}

\begin{axis}[errplot,name=cov-isotropic, ignore legend]
\addplot table[x=d, y=err] {plot_data/excess-cov-isotropic-filter-err.txt};
\addplot table[x=d, y=err] {plot_data/excess-cov-isotropic-LRV-err.txt};
\addplot table[x=d, y=err] {plot_data/excess-cov-isotropic-noisy-emp-err.txt};
\addplot table[x=d, y=err] {plot_data/excess-cov-isotropic-prune-err.txt};
\addplot table[x=d, y=err] {plot_data/excess-cov-isotropic-ransac-err.txt};
\end{axis}

\begin{axis}[errplot,name=cov-isotropic2,at=(cov-isotropic.south west), anchor = north west, yshift=-\plotyspacing, ignore legend]
\addplot table[x=d, y=err] {plot_data/excess-cov-isotropic-filter-err.txt};
\addplot table[x=d, y=err] {plot_data/excess-cov-isotropic-LRV-err.txt};
\addplot table[x=d, y=err] {plot_data/excess-cov-isotropic-noisy-emp-err.txt};
\addplot table[x=d, y=err] {plot_data/excess-cov-isotropic-prune-err.txt};
\end{axis}

\begin{axis}[errplot,name=cov-skew,at=(cov-isotropic.north east),anchor=north west, xshift=\plotxspacing, legend style={at=(cov-isotropic2.south west),anchor=north west, xshift=-1.65 \plotwidth, yshift=-\plotyspacing}]
\addplot table[x=d, y=err] {plot_data/excess-cov-filter-err2.txt};
\addplot table[x=d, y=err] {plot_data/excess-cov-LRV-err2.txt};
\addplot table[x=d, y=err] {plot_data/excess-cov-noisy-samp-err2.txt};
\addplot table[x=d, y=err] {plot_data/excess-cov-prune-err2.txt};
\addplot table[x=d, y=err] {plot_data/excess-cov-ransac-err2.txt};
\legend{Filtering ,\texttt{LRVCov}, Sample covariance w/ noise, Pruning, RANSAC}
\end{axis}

\begin{axis}[errplot,name=cov-skew2,at=(cov-skew.south west),anchor=north west, yshift=-\plotyspacing]
\addplot table[x=d, y=err] {plot_data/excess-cov-filter-err2.txt};
\addplot table[x=d, y=err] {plot_data/excess-cov-LRV-err2.txt};
\end{axis}

\node [at=(cov-isotropic.north east),anchor=south west,xshift=-2cm,yshift=.2cm] {Isotropic};
\node [at=(cov-skew.north east),anchor=south west,xshift=-2cm,yshift=.2cm] {Skewed};

\end{tikzpicture}
\caption{Experiments with synthetic data for robust covariance estimation: error is reported against dimension (lower is better).
The error is excess Mahalanobis error over the sample covariance without noise (the benchmark).
We plot performance of our algorithm, \texttt{LRVCov}, empirical covariance with noise, pruning, and RANSAC. We report two settings: one where the true covariance is isotropic (left column), and one where the true covariance is very skewed (right column). In both, the latter three algorithms have substantially larger error than ours or \texttt{LRVCov}. On the bottom, we restrict our attention to our algorithm and \texttt{LRVCov}. The error achieved by \texttt{LRVCov} is quite good, but ours is better. In particular, our excess error is $4$ orders of magnitude smaller than \texttt{LRVCov}'s in high dimensions.}
\label{fig:synthetic-cov}
\end{figure}
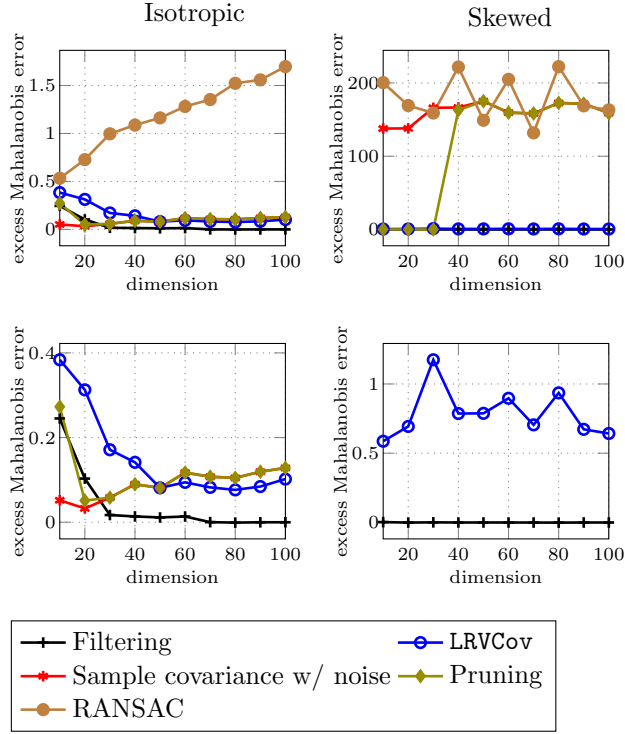

\section{Experiments}

We performed an empirical evaluation of the above algorithms on synthetic and real data sets with and without synthetic noise.
All experiments were done on a laptop computer with a 2.7 GHz Intel Core i5 CPU and 8 GB of RAM.
The focus of this evaluation was on statistical accuracy, not time efficiency.
In this measure, our algorithm performs the best of all algorithms we tried.
In all synthetic trials, our algorithm consistently had the smallest error.
In fact, in some of the synthetic benchmarks, our error was orders of magnitude better than any other algorithms.
In the semi-synthetic benchmark, our algorithm also (arguably) performs the best, though there is no way to tell for sure, since there is no ground truth. 
We also note that despite not optimizing our code for runtime, the runtime of our algorithm is always comparable, and in many cases, better than the alternatives which provided comparable error.
Code of our implementation is available at \url{https://github.com/hoonose/robust-filter}.

\subsection{Synthetic Data}
Experiments with synthetic data allow us to verify the error guarantees and the sample complexity rates proven in Section \ref{sec:optimal-sample} for unknown mean and unknown covariance.
In both cases, the experiments validate the accuracy and usefulness of our algorithm, almost exactly matching the best rate without noise.

\paragraph{Unknown mean} The results of our synthetic mean experiment are shown in Figure \ref{fig:synthetic-mean}.
In the synthetic mean experiment, we set $\varepsilon = 0.1$, and for dimension $d = [100, 150, \ldots, 400]$, we generate $n = \frac{10 d}{\varepsilon^2}$ samples, where a $(1 - \varepsilon)$-fraction come from $\normal (\mu, I)$, and an $\varepsilon$ fraction come from a noise distribution.
Our goal is to produce an estimator which minimizes the $\ell_2$ error the estimator has to the truth.
As a baseline, we compute the error that is achieved by only the uncorrupted sample points.
This error will be used as the gold standard for comparison, since in the presence of error, this is roughly the best one could do even if all the noise points were identified exactly.\footnote{We note that it is possible that an estimator may achieve slightly better error than this baseline.}

On this data, we compared the performance of our Filter algorithm to that of (1) the empirical mean of all the points, (2) a trivial pruning procedure, (3) the geometric median of the data, (4) a RANSAC-based mean estimation algorithm, and (5) a recently proposed robust estimator for the mean due to \cite{LaiRV16}, which we will call \texttt{LRVMean}.
For (5), we use the implementation available in their Github.\footnote{\footnotesize{\url{https://github.com/kal2000/AgnosticMean\\AndCovarianceCode}}}
In Figure \ref{fig:synthetic-mean}, the x-axis indicates the dimension of the experiment, 
and the y-axis measures the $\ell_2$ error of our estimated mean minus the $\ell_2$ error of the empirical mean 
of the true samples from the Gaussian, i.e., the excess error induced over the sampling error.

We tried various noise distributions, and found that the same qualitative pattern arose for all of them.
In the reported experiment, our noise distribution was a mixture of two binary product distributions, 
where one had a couple of large coordinates (see Section~\ref{sec:exp-setup} for a detailed description).
For all (nontrivial) error distributions we tried, we observed that indeed the empirical mean, pruning, geometric median, and RANSAC 
all have error which diverges as $d$ grows, as the theory predicts.
On the other hand, both our algorithm and \texttt{LRVMean} have markedly smaller error as a function of dimension.
Indeed, our algorithm's error is almost identical to that of the empirical mean of the uncorrupted sample points.


\paragraph{Unknown covariance}
The results of our synthetic covariance experiment are shown in Figure \ref{fig:synthetic-cov}.
Our setup is similar to that for the synthetic mean.
Since both our algorithm and \texttt{LRVCov} require access to fourth moment objects, we ran into issues with limited memory on machines.
Thus, we could not perform experiments at as high a dimension as for the unknown mean setting, and we could not use as many samples.
We set $\eps = 0.05$, and for dimension $d = [10, 20, \ldots, 100]$, we generate $n = \frac{0.5 d}{\varepsilon^2}$ samples, 
where a $(1 - \varepsilon)$-fraction come from $\normal (0, \Sigma)$, and an $\varepsilon$ fraction come from a noise distribution.
We measure distance in the natural affine invariant way, namely, the Mahalanobis distance induced by $\Sigma$ to the identity: 
$\mathrm{err} (\Sigmahat) = \| \Sigma^{-1/2} \Sigmahat \Sigma^{-1/2} - I \|_F$.
As explained above, this is the right affine-invariant metric for this problem.
As before, we use the empirical error of only the uncorrupted data points as a benchmark.

On this corrupted data, we compared the performance of our Filter algorithm to that of 
(1) the empirical covariance of all the points, (2) a trivial pruning procedure, (3) a RANSAC-based minimal volume ellipsoid (MVE) algorithm, 
and (5) a recently proposed robust estimator for the covariance due to \cite{LaiRV16}, which we will call \texttt{LRVCov}.
For (5), we again obtained the implementation from their Github repository.

We tried various choices of $\Sigma$ and noise distribution.
Figure \ref{fig:synthetic-cov} shows two choices of $\Sigma$ and noise.
Again, the x-axis indicates the dimension of the experiment and the y-axis indicates the estimator's excess Mahalanobis error over the sampling error.
In the left figure, we set $\Sigma = I$, and our noise points are simply all located at the all-zeros vector.
In the right figure, we set $\Sigma = I + 10 e_1 e_1^T$, where $e_1$ is the first basis vector, and our noise distribution is a somewhat more complicated distribution, which is similarly spiked, but in a different, random, direction. We formally define this distribution in Section~\ref{sec:exp-setup}.
For all choices of $\Sigma$ and noise we tried, the qualitative behavior of our algorithm and \texttt{LRVCov} was unchanged.
Namely, we seem to match the empirical error without noise up to a very small slack, for all dimensions.
On the other hand, the performance of empirical mean, pruning, and RANSAC varies widely with the noise distribution.
The performance of all these algorithms degrades substantially with dimension, and their error gets worse as we increase the skew of the underlying data.
The performance of \texttt{LRVCov} is the most similar to ours, but again is worse by a large constant factor.
In particular, our excess risk was on the order of $10^{-4}$ for large $d$, for both experiments, whereas the excess risk achieved by \texttt{LRVCov} was in all cases a constant between $0.1$ and $2$.

\paragraph{Discussion} These experiments demonstrate that our statistical guarantees are in fact quite strong.
In particular, since our excess error is almost zero (and orders of magnitude smaller than other approaches), this suggests that our sample complexity is indeed close to optimal, since we match the rate without noise, and that the constants and logarithmic factors in the theoretical recovery guarantee are often small or non-existent.

\begin{centering}
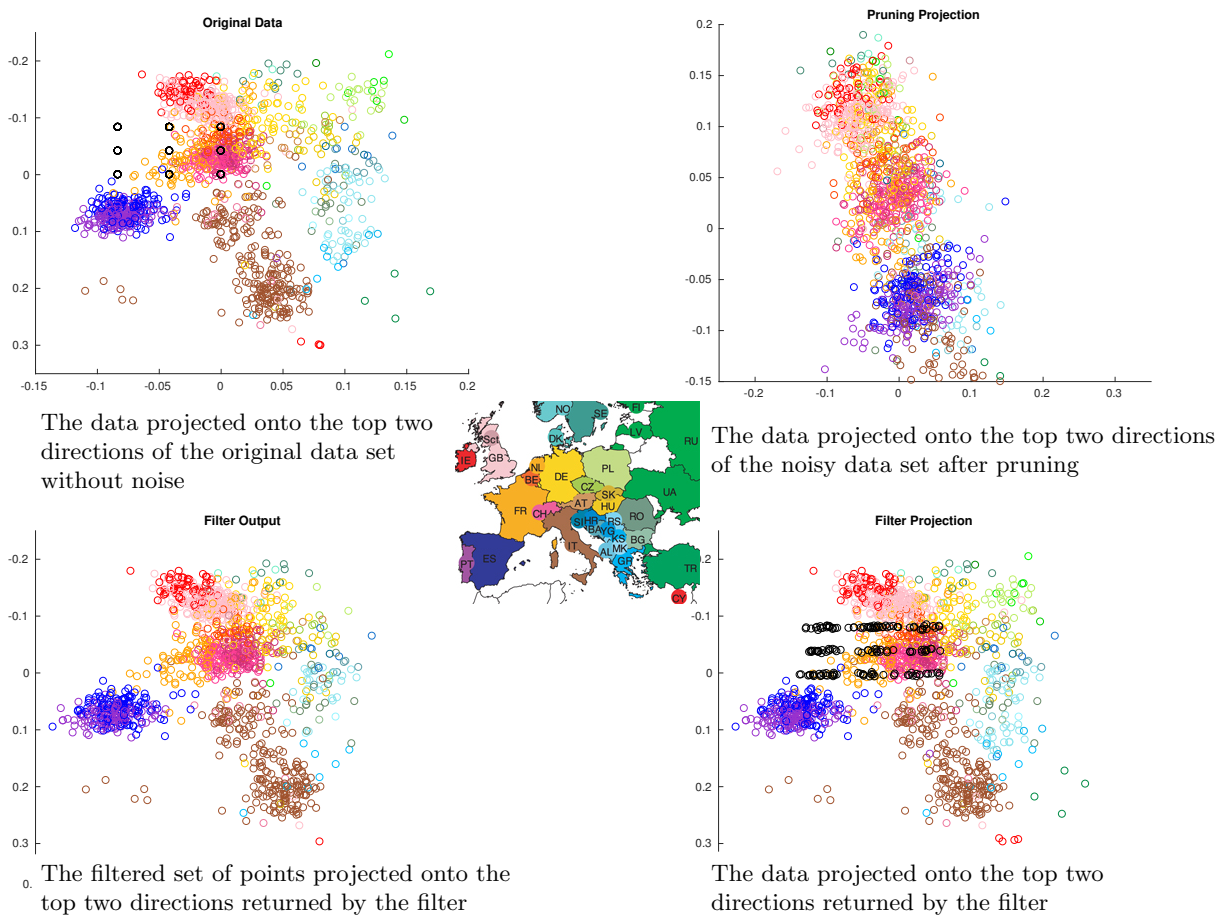
\begin{figure*}[h!]
\begin{tikzpicture}[
block/.style={
  fill=white, 
  text width=0.7*\columnwidth, 
  anchor=west,
  minimum height=1cm,
  rounded corners 
  }, 
font=\small
]

\node[inner sep=0pt] (original) at (.2\textwidth,0.03\textheight)
    {\includegraphics[width=.45\textwidth]{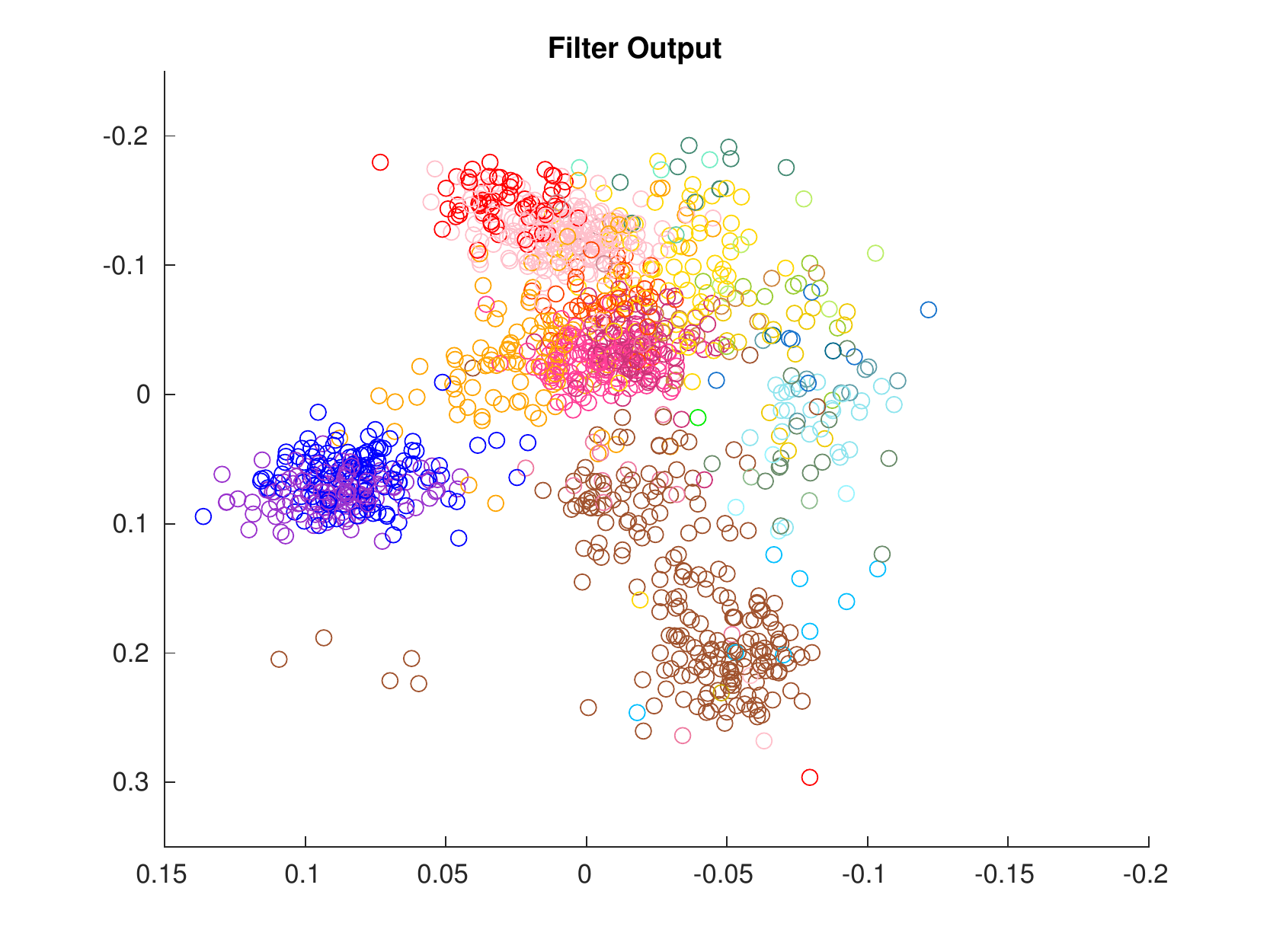}};
    \node[block,align=left] (originalcap) at ( .03 \textwidth, -.075\textheight)
  (step1)
  {The filtered set of points projected onto the \\ top two directions returned by the filter};
\node[inner sep=0pt] (pruning) at (.75\textwidth, 0.03\textheight)
    {\includegraphics[width=.45\textwidth]{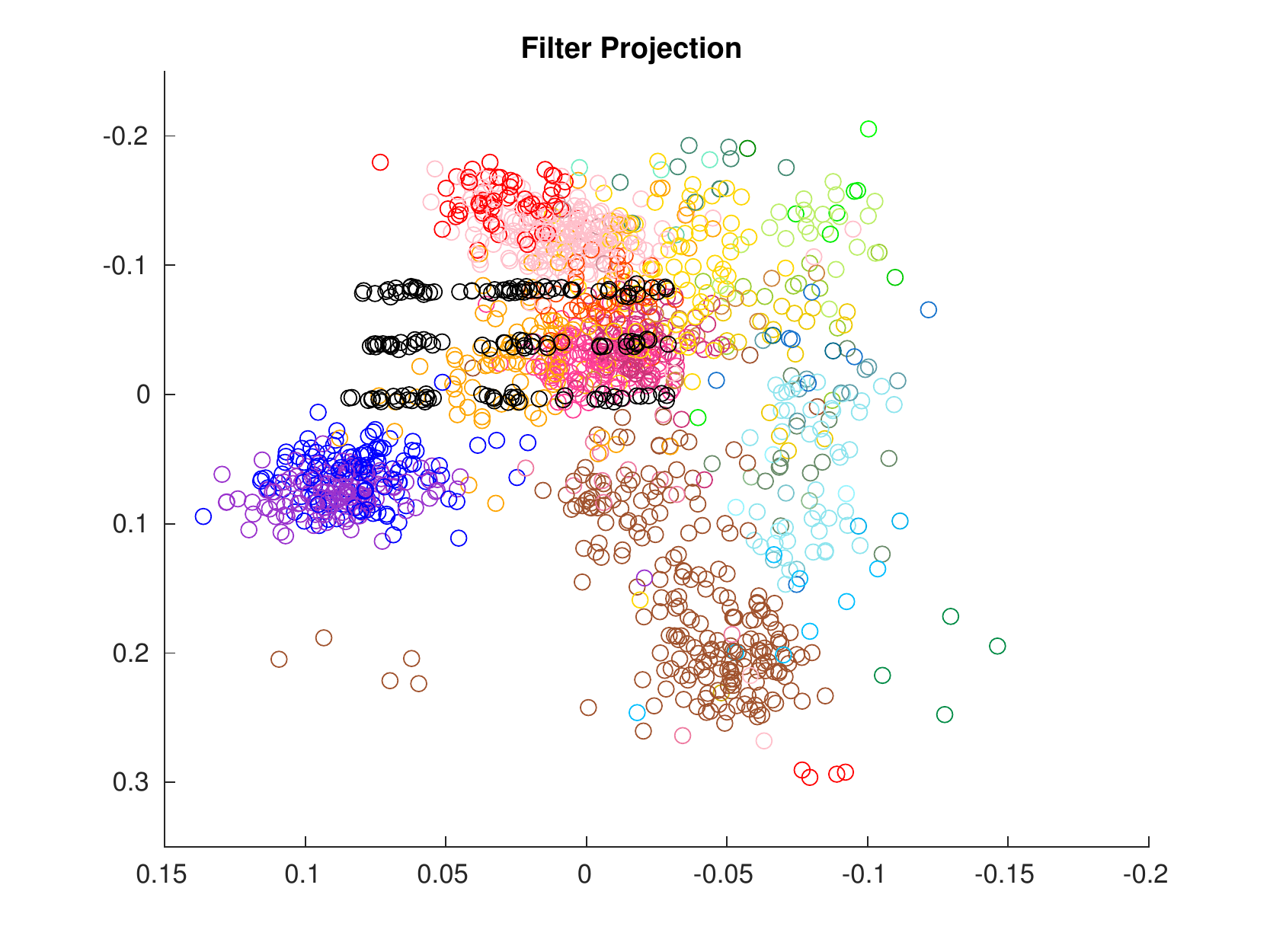}};
   \node[block,align=left] (originalcap) at ( .57 \textwidth, -.075\textheight)
  (step1)
  {The data projected onto the top two \\ directions returned by the filter};
\node[inner sep=0pt] (pruning) at (.2\textwidth, .32\textheight)
    {\includegraphics[width=.45\textwidth]{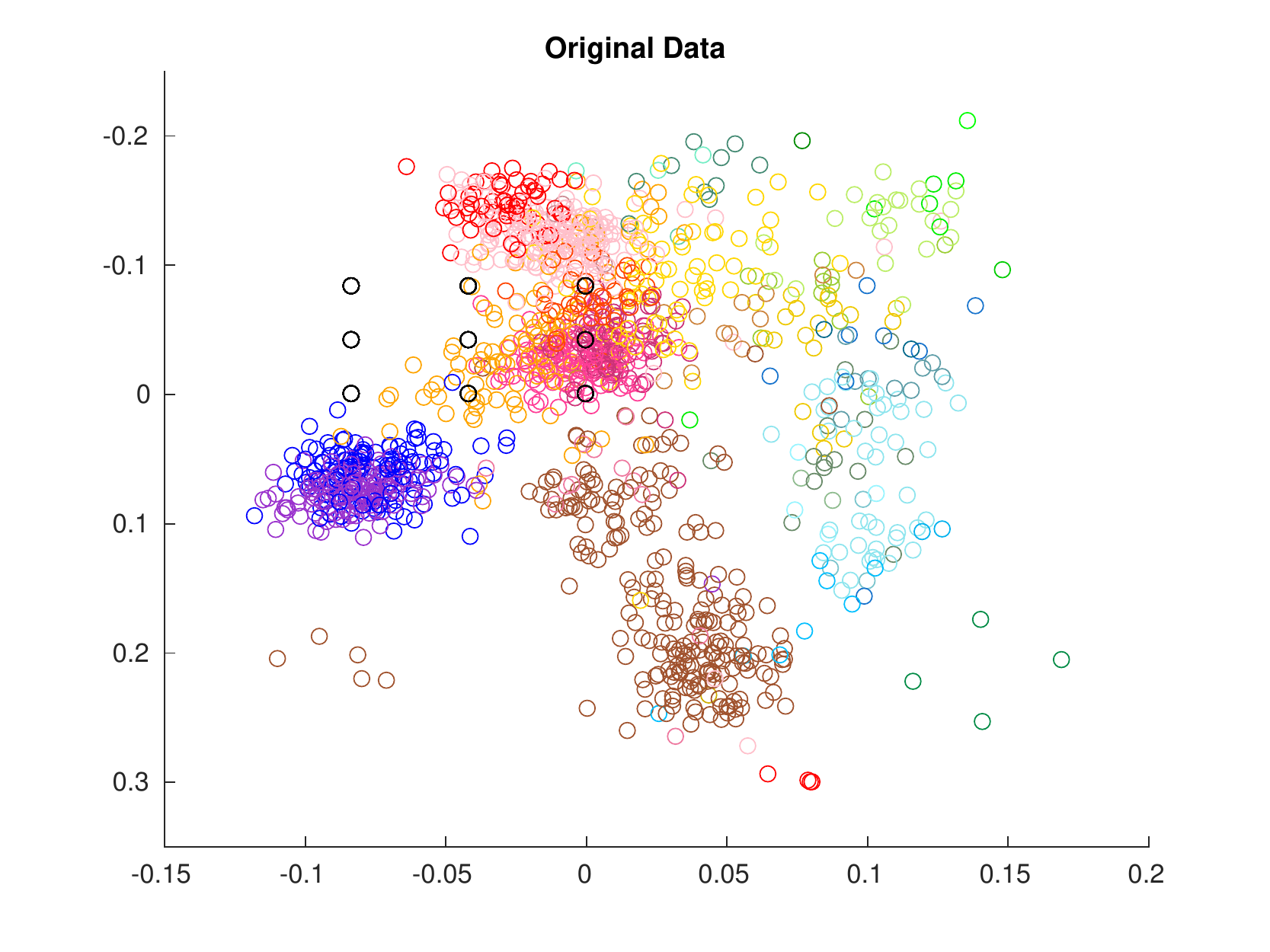}};
    \node[block,align=left] (originalcap) at ( .03 \textwidth, .18\textheight)
  (step1)
  {The data projected onto the top two \\ directions of the original data set \\ without noise};
\node[inner sep=0pt] (pruning) at (.75 \textwidth, .32\textheight)
    {\includegraphics[width=.45\textwidth]{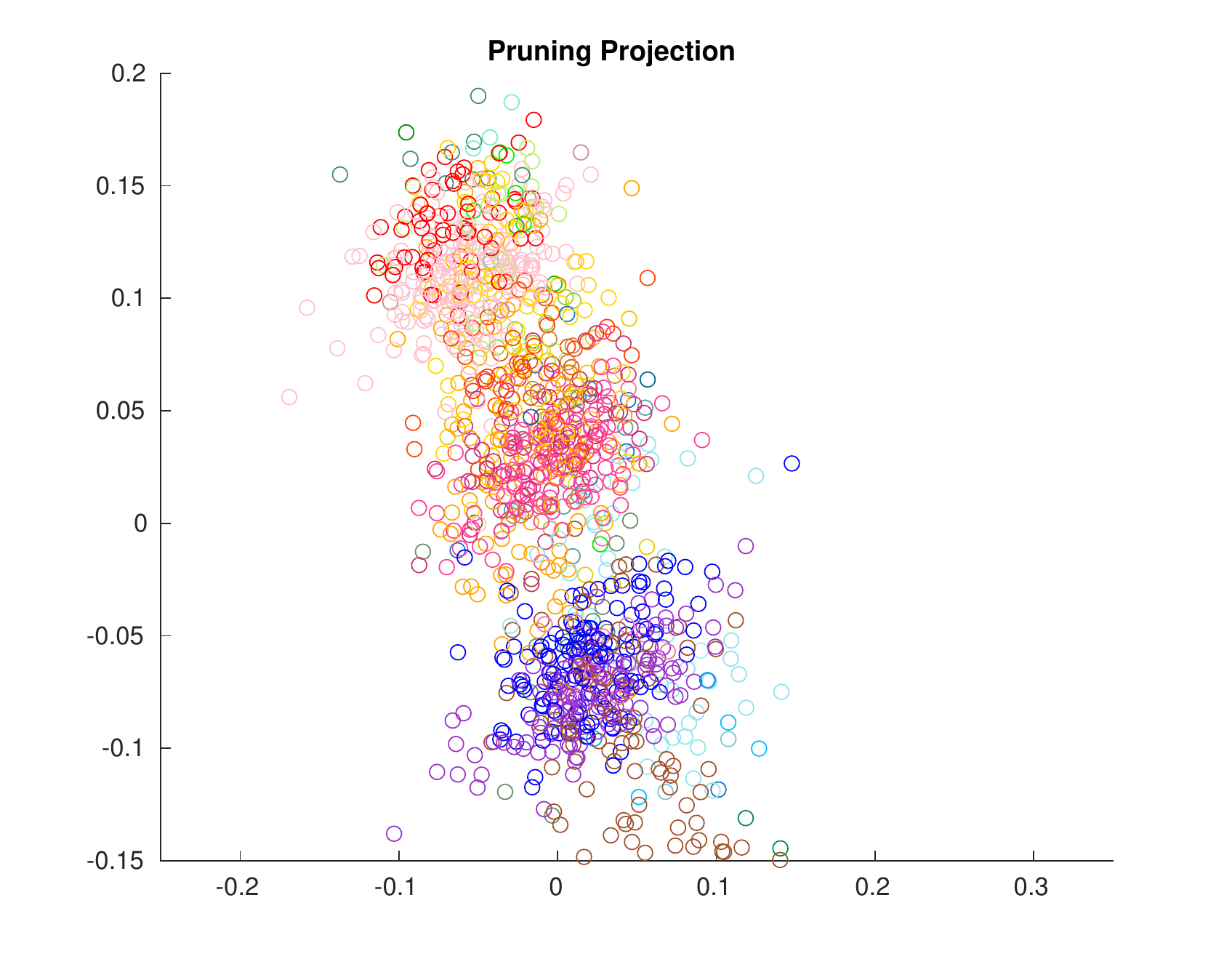}};
   \node[block,align=left] (originalcap) at ( .57 \textwidth, .18\textheight)
  (step1)
  {The data projected onto the top two directions \\ of the noisy data set after pruning};
    
\node[inner sep=0pt] (pruning) at (.47 \textwidth, .15\textheight)
   {\includegraphics[scale=0.5]{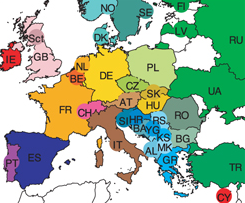}};
\end{tikzpicture}
\vspace{-0.8cm}
\caption{Experiments with semi-synthetic data: given the real genetic data from  \cite{novembre2008genes}, projected down to 20-dimensions, and with added noise. The colors indicate the country of origin of the person, and match the colors of the countries in the map of Europe in the center. Black points are added noise. The top left plot is the original plot from \cite{novembre2008genes}. We (mostly) recover Europe in the presence of noise whereas naive methods do not.}
\label{fig:europe}
\end{figure*}
\end{centering}

\subsection{Semi-synthetic Data}
To demonstrate the efficacy of our method on real data, we revisit the famous study of \cite{novembre2008genes}.
In this study, the authors investigated data collected as part of the Population Reference Sample (POPRES) project.
This dataset consists of the genotyping of thousands of individuals using the Affymetrix 500K single nucleotide polymorphism (SNP) chip.
The authors pruned the dataset to obtain the genetic data of over 1387 European individuals, annotated by their country of origin.
Using principal components analysis, they produce a two-dimensional summary of the genetic variation, which bears a striking resemblance to the map of Europe.

Our experimental setup is as follows.
While the original dataset is very high dimensional, we use a 20 dimensional version of the dataset as found in the authors' GitHub\footnote{\url{https://github.com/NovembreLab/Novembre_etal_2008_misc}}.
We first randomly rotate the data, as then 20 dimensional data was diagonalized, and the high dimensional data does not follow such structure.
We then add an additional $\frac{\epsilon}{1 - \epsilon}$ fraction of points (so that they make up an $\epsilon$-fraction of the final points).
These added points were discrete points, following a simple product distribution (see Section~\ref{sec:exp-setup} for full details).
We used a number of methods to obtain a covariance matrix for this dataset, and we projected the data onto the top two singular vectors of this matrix.
In Figure \ref{fig:europe}, we show the results when we compare our techniques to pruning.
In particular, our output was able to more or less reproduce the map of Europe, whereas pruning fails to.
In Section~\ref{sec:exp-others}, we also compare our result with a number of other techniques, including those we tested against in the unknown covariance experiments, and other robust PCA techniques.
The only alternative algorithm which was able to produce meaningful output was \texttt{LRVCov}, which produced output that was similar to ours, but which produced a map which was somewhat more skewed.
We believe that our algorithm produces the best picture.

In Figure \ref{fig:europe}, we also display the actual points which were output by our algorithm's Filter.
While it manages to remove most of the noise points, it also seems to remove some of the true data points, particularly those from Eastern Europe and Turkey.
We attribute this to a lack of samples from these regions, and thus one could consider them as outliers to a dataset consisting of Western European individuals.
For instance, Turkey had 4 data points, so it seems quite reasonable that any robust algorithm would naturally consider these points outliers.

\paragraph{Discussion}
We view our experiments as a proof of concept demonstration that our techniques can be useful in real world exploratory data analysis tasks, particularly those in high-dimensions. 
Our experiments reveal that a minimal amount of noise can completely disrupt a data analyst's ability to notice an interesting phenomenon, thus limiting us to only very well-curated data sets.
But with robust methods, this noise does not interfere with scientific discovery, and we can still recover interesting patterns which otherwise would have been obscured by noise.

\section*{Acknowledgments}
We would like to thank Simon Du and Lili Su for helpful comments on a previous version of this work.
\bibliographystyle{alpha}
\bibliography{biblio}
\appendix

\section{Omitted Details from Section~3} \label{sec:app-theory}

\subsection{Robust Mean Estimation for Sub-Gaussian Distributions} \label{sec:filter-subgaussian}

In this section, we use our filter technique to give a {\em near sample-optimal} computationally efficient algorithm 
to robustly estimate the mean of a sub-gaussian density with a known covariance matrix, thus proving Theorem~\ref{thm:filter-gaussian-mean}.

We emphasize that the algorithm and its analysis is essentially identical to the filtering algorithm given in Section~8.1 of~\cite{DKKLMS} 
for the case of a Gaussian $\mathcal{N}(\mu, I)$.
The only difference is a weaker definition of the ``good set of samples'' (Definition~\ref{def:good-set}) and a simple
concentration argument (Lemma~ \ref{lem:random-good-gaussian-mean}) showing that a random set of uncorrupted 
samples of the appropriate size is good with high probability. Given these, the analysis of this subsection follows straightforwardly
from the analysis in Section~8.1 of~\cite{DKKLMS} by plugging in the modified parameters. For the sake of completeness,
we provide the details below.

We start by formally defining sub-gaussian distributions:

\begin{definition} \label{def:subg}
A distribution $P$ on $\R$ with mean $\mu$, is sub-gaussian with parameter $\nu>0$  if 
$$\E_{X \sim P} \left[\exp(\lambda (X-\mu))\right] \leq \exp(\nu\lambda^2/2)$$ for all $\lambda \in \R$. 
A distribution $P$ on $\R^d$ with mean vector $\mu$ is sub-gaussian with parameter $\nu>0$, if  for all unit vectors $v$, 
the one-dimensional random variable $v \cdot X$, $X \sim P$, is sub-gaussian with parameter $\nu$.
\end{definition}

\noindent We will use the following simple fact about the concentration of sub-gaussian random variables:

\begin{fact} \label{fact:tail-bound} 
If $P$ is sub-gaussian on $\R^d$ with mean vector $\mu$ and parameter $\nu>0$, 
then for any unit vector $v \in \R^d$ we have that $\Pr_{X \sim P}\left[|v \cdot (X-\mu)| \geq T \right] \leq \exp(-T^2/2\nu)$.
\end{fact}

The following theorem is a high probability version of Theorem~\ref{thm:filter-gaussian-mean}:

\begin{theorem} \label{thm:filter-gaussian-mean-app}
Let $G$ be a sub-gaussian distribution on $\R^d$ with parameter $\nu=\Theta(1)$, 
mean $\mu^G$, covariance matrix $I$, and $\eps, \tau > 0$.
Let $S'$ be an $\eps$-corrupted set of samples from $G$ of size
\new{$\Omega((d/\eps^2) \poly\log(d/\eps\tau))$}. 
There exists an efficient algorithm that, on input $S'$ and $\eps>0$, returns a mean vector $\wh{\mu}$
so that with probability at least $1-\tau$ we have $\|\wh{\mu}-\mu^{G}\|_2 = O(\eps\sqrt{\log(1/\eps)}).$
\end{theorem}



\medskip

\noindent {\bf Notation.} We will denote $\mu^S = \frac{1}{|S|}\sum_{X\in S} X$ and 
$M_S=\frac{1}{|S|}\sum_{X\in S} (X-\mu^G)(X-\mu^G)^T$ for the sample mean and 
modified sample covariance matrix of the set $S$.

\medskip

We start by defining our modified notion of good sample, i.e, 
a set of conditions on the uncorrupted set of samples under which our algorithm will succeed.

\begin{definition} \label{def:good-set}
Let $G$ be an identity covariance sub-gaussian in $d$ dimensions with mean $\mu^G$ and covariance matrix $I$ and $\eps,\tau >0$.
We say that a multiset $S$ of elements in $\R^d$ is {\em $(\eps,\tau)$-good with respect to $G$}
if the following conditions are satisfied:
\begin{itemize}
\item[(i)] For all $x \in S$ we have $\|x-\mu^G\|_2 \leq O(\sqrt{d \log (|S|/\tau)})$.

\item[(ii)]  For every affine function $L:\R^d \to \R$ \new{such that $L(x) = v \cdot (x-\mu^G)-T$, $\|v\|_2=1$,}
we have that 
$\left|\Pr_{X \in_u S}[L(X) \ge 0] - \Pr_{X\sim G}[L(X) \ge 0] \right| \leq \frac{\eps}{\new{T^2\log\left(d \log(\frac{d}{\eps\tau})\right)}} \;.$

\item[(iii)]  We have that $\|\mu^S - \mu^G \|_2\leq \eps.$

\item[(iv)]  We have that $\left\|M_S- I \right\|_2 \leq \new{\eps}.$
\end{itemize}
\end{definition}

We show in the following subsection that a sufficiently large set of independent samples from $G$
is $(\eps,\tau)$-good (with respect to $G$) with high probability. Specifically, we prove:

\begin{lemma} \label{lem:random-good-gaussian-mean}
Let $G$ be sub-gaussian distribution with parameter $\nu=\Theta(1)$ and with identity covariance, and $\eps, \tau>0.$
If the multiset $S$ is obtained by taking \new{$\Omega( (d/\eps^2) \poly\log(d/\eps\tau))$} independent samples from $G,$
it is $(\eps,\tau)$-good with respect to $G$ with probability at least $1-\tau .$
\end{lemma}

We require the following definition that quantifies the extent to which a multiset has been corrupted:

\begin{definition} \label{def:Delta-G}
Given finite multisets $S$ and $S'$ we let $\Delta(S,S')$
be the size of the symmetric difference of $S$ and $S'$ divided by the cardinality of $S.$
\end{definition}

The starting point of our algorithm will be a simple~\textsc{NaivePrune} routine (Section~4.3.1 of ~\cite{DKKLMS}) 
that removes obvious outliers, i.e., points which are far from the mean. 
Then, we iterate the algorithm whose performance guarantee is given by the following:

\begin{proposition} \label{prop:filter-gaussian-mean}
Let $G$ be a sub-gaussian distribution on $\R^d$ with parameter $\nu=\Theta(1)$,  mean $\mu^G$, covariance matrix $I$, $\eps>0$ be sufficiently  small and $\tau > 0$.
Let $S$ be an $(\eps,\tau)$-good set with respect to $G$.
Let $S'$ be any multiset with $\Delta(S,S') \leq 2\eps$ and for any $x,y \in S'$, $\|x-y\|_2 \leq  O(\sqrt{d \log(d/\eps\tau)})$.
There exists a polynomial time algorithm \textsc{Filter-Sub-Gaussian-Unknown-Mean}
that, given $S'$ and $\eps>0,$ returns one of the following:
\begin{enumerate}
\item[(i)]  A mean vector $\wh{\mu}$ such that $\|\wh{\mu}-\mu^{G}\|_2 = O(\eps\sqrt{\log(1/\eps)}).$
\item[(ii)] A multiset $S'' \subseteq S'$ such that $\Delta(S,S'') \leq \Delta(S,S') - \eps/\new{\alpha}$, where
\new{$\alpha \eqdef d\log(d/\eps\tau)\log(d \log(\frac{d}{\eps\tau}))$.}
\end{enumerate}
\end{proposition}

We start by showing how Theorem~\ref{thm:filter-gaussian-mean-app}  follows
easily from Proposition~\ref{prop:filter-gaussian-mean}.

\begin{proof}[Proof of Theorem~\ref{thm:filter-gaussian-mean-app}]
By the definition of $\Delta(S, S'),$ since $S'$ has been obtained from $S$
by corrupting an $\eps$-fraction of the points in $S,$ we have that
$\Delta(S, S') \le 2\eps.$ By Lemma~\ref{lem:random-good-gaussian-mean}, the set $S$ of uncorrupted samples
is $(\eps,\tau)$-good with respect to $G$ with probability at least $1-\tau.$
We henceforth condition on this event.

Since $S$ is $(\eps,\tau)$-good, all $x \in S$ 
have $\|x-\mu^G\|_2 \leq O(\sqrt{d \log |S|/\tau})$. 
Thus, the \textsc{NaivePrune} procedure does not remove from $S'$ any member of $S$. 
Hence, its output, $S''$, has $\Delta(S, S'') \leq \Delta(S, S')$ 
and for any $x \in S''$, there is a $y \in S$ with $\|x-y\|_2  \leq O(\sqrt{d \log |S|/\tau})$. 
By the triangle inequality, for any $x,z  \in S''$, $\|x-z\|_2 \leq O(\sqrt{d \log |S|/\tau})= O(\sqrt{d \log (d/\eps\tau}))$.

Then, we iteratively apply the \textsc{Filter-Sub-Gaussian-Unknown-Mean} procedure 
of Proposition~\ref{prop:filter-gaussian-mean} until it terminates
returning a mean vector $\mu$ with $\|\wh{\mu}-\mu^{G}\|_2 = O(\eps\sqrt{\log(1/\eps)}).$
We claim that we need at most \new{$O(\alpha)$} iterations for this to happen.
Indeed, the sequence of iterations results in a sequence of sets $S_i',$
so that $\Delta(S,S_i') \leq \Delta(S,S')  - i \cdot \eps/\new{\alpha}.$
Thus, if we do not output the empirical mean in the first \new{$2\alpha$} iterations, 
in the next iteration there are no outliers left and the algorithm terminates 
outputting the sample mean of the remaining set.
\end{proof}

\subsubsection{Algorithm \textsc{Filter-Sub-Gaussian-Unknown-Mean}: Proof of Proposition~\ref{prop:filter-gaussian-mean}}

In this subsection, we describe the efficient algorithm establishing
Proposition~\ref{prop:filter-gaussian-mean} and prove its correctness.
Our algorithm calculates the empirical mean vector $\mu^{S'}$ and empirical covariance matrix $\Sigma$.
If the matrix $\Sigma$ has no large eigenvalues, it returns $\mu^{S'}.$
Otherwise, it uses the eigenvector $v^{\ast}$ corresponding to the maximum magnitude eigenvalue of $\Sigma$
and the mean vector $\mu^{S'}$ to define a filter.
Our efficient filtering procedure is presented in detailed pseudocode below.

\begin{algorithm}
\begin{algorithmic}[1]
\Procedure{Filter-Sub-Gaussian-Unknown-Mean}{$S',\eps,\tau$}
\INPUT A multiset $S'$ such that there exists an $(\eps,\tau)$-good $S$ with $\Delta(S, S') \le 2\eps$
\OUTPUT Multiset $S''$ or mean vector $\wh{\mu}$ satisfying Proposition~\ref{prop:filter-gaussian-mean}
\State Compute the sample mean $\mu^{S'}=\E_{X\in_u S'}[X]$ and the sample covariance matrix $\Sigma$ ,
i.e., $\Sigma = (\Sigma_{i, j})_{1 \le i, j \le d}$
with $\Sigma_{i,j} = \E_{X\in_u S'}[(X_i-\mu^{S'}_i) (X_j-\mu^{S'}_j)]$.
\State  Compute approximations for the largest absolute eigenvalue of $\Sigma - I$, $\lambda^{\ast} := \|\Sigma - I\|_2,$
and the associated unit eigenvector $v^{\ast}.$

\If {$\|\Sigma - I\|_2 \leq O(\eps \log (1/\eps)),$}
 \textbf{return} $\mu^{S'}.$ \label{step:bal-small-G}
\EndIf
\State \label{step:bal-large-G}  Let $\delta := 3 \sqrt{ \eps  \|\Sigma - I\|_2}.$ Find $T>0$ such that
$$
\Pr_{X\in_u S'}\left[|v^{\ast} \cdot (X-\mu^{S'})|>T+\delta \right] > 8\exp(-T^2/2\nu)+8 \frac{\eps}{\new{T^2\log\left(d \log(\frac{d}{\eps\tau})\right)}}.
$$
\State \label{step:gaussian-mean-filter} \textbf{return} the multiset $S''=\{x\in S': |v^{\ast} \cdot (x-\mu^{S'}) | \leq T+\delta\}$.

\EndProcedure
\end{algorithmic}
\caption{Filter algorithm for a sub-gaussian with unknown mean and identity covariance}
\label{alg:filter-Gaussian-mean}
\end{algorithm}

\subsubsection{Proof of Correctness of \textsc{Filter-Sub-Gaussian-Unknown-Mean}} \label{ssec:L2-setup-G}
By definition, there exist disjoint multisets $L,E,$ of points in $\R^d$, where $L \subset S,$
such that $S' = (S\setminus L) \cup E.$ With this notation, we can write $\Delta(S,S')=\frac{|L|+|E|}{|S|}.$
Our assumption $\Delta(S,S') \le 2\eps$ is equivalent to $|L|+|E| \le 2\eps \cdot |S|,$ and the definition of $S'$
directly implies that $(1-2\eps)|S| \le  |S'| \le (1+2\eps) |S|.$ Throughout the proof, we assume that $\eps$
is a sufficiently small constant. 

We define $\mu^G,\mu^S,\mu^{S'},\mu^L,$ and $\mu^E$ to be the means of $G,S,S',L,$ and $E$, respectively.

Our analysis will make essential use of the following matrices:
\begin{itemize}
\item $M_{S'}$ denotes $\E_{X\in_u S'}[(X-\mu^G)(X-\mu^G)^T]$,
\item $M_{S}$ denotes $\E_{X\in_u S}[(X-\mu^G)(X-\mu^G)^T]$,
\item $M_{L}$ denotes $\E_{X\in_u L}[(X-\mu^G)(X-\mu^G)^T]$, and
\item $M_{E}$ denotes $\E_{X\in_u E}[(X-\mu^G)(X-\mu^G)^T]$.
\end{itemize}

Our analysis will hinge on proving the important claim that $\Sigma-I$ is approximately $(|E|/|S'|)M_E$. 
This means two things for us. First, it means that if the positive errors align in some direction 
(causing $M_E$ to have a large eigenvalue), there will be a large eigenvalue in $\Sigma-I$. 
Second, it says that any large eigenvalue of $\Sigma-I$ will correspond to an eigenvalue of $M_E$, 
which will give an explicit direction in which many error points are far from the \new{empirical} mean.

\medskip

\noindent {\bf Useful Structural Lemmas.} We begin by noting that we have concentration 
bounds on $G$ and therefore, on $S$ due to its goodness.
\begin{fact}\label{fact:conc}
Let $w \in \R^d$ be any unit vector, then for any $T>0$,
$\Pr_{X\sim G}\left[|w\cdot(X-\mu^{G})| > T\right] \leq 2 \exp(-T^2/2\nu)$ and
$\Pr_{X\in_u S}\left[|w\cdot(X-\mu^{G})| > T\right] \leq 2 \exp(-T^2/2\nu)+\frac{\eps}{\new{T^2\log\left(d \log(\frac{d}{\eps\tau})\right)}}$.
\end{fact}
\begin{proof}
The first line is Fact \ref{fact:tail-bound}, and the former follows from it using the goodness of $S$.
\end{proof}

By using the above fact, we obtain the following simple claim:

\begin{claim}\label{claim:conc}
Let $w \in \R^d$ be any unit vector, then for any $T>0$, we have that:
$$
\Pr_{X\sim G}[|w\cdot(X-\mu^{S'})| > T + \|\mu^{S'}-\mu^G\|_2] \leq 2 \exp(-T^2/2\nu).
$$
and
$$
\Pr_{X\in_u S}[|w\cdot(X-\mu^{S'})| > T + \|\mu^{S'}-\mu^G\|_2] \leq 2 \exp(-T^2/2\nu)+\frac{\eps}{\new{T^2\log\left(d \log(\frac{d}{\eps\tau})\right)}}.
$$
\end{claim}
\begin{proof}
This follows from Fact \ref{fact:conc} upon noting that $|w\cdot (X-\mu^{S'})| > T +\|\mu^{S'}-\mu^G\|_2$ only if $|w\cdot (X-\mu^G)|>T$.
\end{proof}

We can use the above facts to prove concentration bounds for $L$. In particular, 
we have the following lemma:

\begin{lemma}\label{LBoundCor}
We have that $\|M_L\|_2 = \new{O\left(\log(|S|/|L|) + \eps |S| / |L| \right)}.$
\end{lemma}
\begin{proof}
Since $L \subseteq S$, for any $x \in \R^d$, we have that
\begin{equation} \label{eqn:L-subset-S}
|S| \cdot \Pr_{X \in_u S}(X= x) \geq |L| \cdot \Pr_{X \in_u L}(X= x) \;.
\end{equation}
Since $M_L$ is a symmetric matrix, we have $\|M_L\|_2 = \max_{\|v\|_2=1} |v^T M_L v|.$ So,
to bound $\|M_L\|_2$ it suffices to bound $|v^T M_L v|$ for unit vectors $v.$
By definition of $M_L,$
for any $v \in \R^d$ we have that
$$|v^T M_L v| = \E_{X \in_u L}[|v\cdot (X-\mu^{G})|^2].$$
For unit vectors $v$, the RHS is bounded from above as follows:
\begin{align*}
\E_{X \in_u L}\left[|v\cdot (X-\mu^{G})|^2 \right] & =  2 \int_{0}^{\infty} \Pr_{X\in_u L}\left[|v\cdot(X-\mu^G)|>T\right] T dT\\
& = 2 \int_{0}^{O(\sqrt{d \log(d/\eps\tau)})} \Pr_{X\in_u L}[|v\cdot(X-\mu^G)|>T] T dT \\
& \leq 2 \int_0^{O(\sqrt{d \log(d/\eps\tau)})} \min \left\{ 1,  \frac{|S|}{|L|} \cdot \Pr_{X \in_u S}\left[|v \cdot(X-\mu^{G})|>T\right]  \right\} TdT\\
& \ll \int_0^{4\sqrt{\nu \log(|S|/|L|)}} T dT \\
& + (|S|/|L|) \int_{4\sqrt{\nu \log(|S|/|L|)}}^{O(\sqrt{d \log(d/\eps\tau)})}  \Big( \exp(-T^2/2\nu)+\frac{\eps}{\new{T^2\log\left(d \log(\frac{d}{\eps\tau})\right)}} \Big)  T dT\\
& \ll  \log(|S|/|L|) + \eps \cdot |S|/|L| \;,
\end{align*}
where the second line follows from the fact that $\|v\|_2 = 1$, $L \subset S$, and $S$ satisfies condition (i) of Definition~\ref{def:good-set};
the third line follows from (\ref{eqn:L-subset-S}); and the fourth line follows from Fact~\ref{fact:conc}.
\end{proof}

As a corollary, we can relate the matrices $M_{S'}$ and $M_E$, in spectral norm:

\begin{corollary}\label{MSPrimeBound}
We have that 
$M_{S'} - I = (|E|/|S'|)M_E + O(\eps\log(1/\eps))$,
where the $O(\eps\log(1/\eps))$ term denotes a matrix of spectral norm $O(\eps\log(1/\eps))$.
\end{corollary}
\begin{proof}
By definition, we have that $|S'|M_{S'} = |S|M_S - |L|M_L + |E|M_E$.
Thus, we can write
\begin{align*}
M_{S'} & = (|S|/|S'|) M_S - (|L|/|S'|)M_L + (|E|/|S'|)M_E\\
& = I + O(\eps) + O(\eps \log(1/\eps)) + (|E|/|S'|)M_E \;,
\end{align*}
where the second line uses the fact that $1-2\eps \leq |S|/|S'| \leq 1+2\eps$, 
the goodness of $S$ (condition (iv) in Definition~\ref{def:good-set}), and Lemma~\ref{LBoundCor}.
Specifically, Lemma~\ref{LBoundCor} implies that $(|L|/|S'|) \|M_L\|_2 = O(\eps \log(1/\eps))$.
Therefore, we have that
$$
M_{S'} = I + (|E|/|S'|)M_E + O(\eps\log(1/\eps)) \;,
$$
as desired.
\end{proof}

We now establish a similarly useful bound on the difference between the mean vectors:

\begin{lemma}\label{meansBoundLem}
We have that $\mu^{S'}-\mu^G = (|E|/|S'|)(\mu^E-\mu^G) + O(\eps\sqrt{\log(1/\eps)})$,
where the $O(\eps\sqrt{\log(1/\eps)})$ term denotes a vector with $\ell_2$-norm at most $O(\eps\sqrt{\log(1/\eps)})$.
\end{lemma}
\begin{proof}
By definition, we have that
\begin{align*}
|S'|(\mu^{S'}-\mu^G) = |S|(\mu^S-\mu^G) - |L|(\mu^L-\mu^G) + |E|(\mu^E-\mu^G).
\end{align*}
Since $S$ is a good set, by condition (iii) of Definition~\ref{def:good-set}, we have $\|\mu^S-\mu^G\|_2 = O(\eps)$. 
Since $1-2\eps \leq |S|/|S'| \leq 1+2\eps$, it follows that $(|S|/|S'|)\|\mu^S-\mu^G\|_2 = O(\eps)$.
Using the valid inequality $\|M_L\|_2 \geq \|\mu^L-\mu^G\|_2^2$ and Lemma~\ref{LBoundCor}, we obtain that
$\|\mu^L-\mu^G\|_2 \leq O\left(\sqrt{\log(|S|/|L|)} + \sqrt{\eps |S| / |L|} \right)$. Therefore, 
$$(|L|/|S'|)  \|\mu^L-\mu^G\|_2 \leq O\left( (|L|/|S|) \sqrt{\log(|S|/|L|)} + \sqrt{\eps |L| / |S|} \right) = O(\eps\sqrt{\log(1/\eps)}) \;.$$ 
In summary,
$$
\mu^{S'}-\mu^G = (|E|/|S'|)(\mu^E-\mu^G) + O(\eps\sqrt{\log(1/\eps)}) \;,
$$
as desired. This completes the proof of the lemma.
\end{proof}

By combining the above, we can conclude that $\Sigma-I$ is approximately proportional to $M_E$.
More formally, we obtain the following corollary:

\begin{corollary}\label{MApproxCor-G}
We have 
$\Sigma-I = (|E|/|S'|)M_E + O(\eps\log(1/\eps))+O(|E|/|S'|)^2\|M_E\|_2$, where the additive terms denote matrices
of appropriately bounded spectral norm.
\end{corollary}
\begin{proof}
By definition, we can write
$\Sigma-I = M_{S'}-I-(\mu^{S'}-\mu^G)(\mu^{S'}-\mu^G)^T.$
Using Corollary~\ref{MSPrimeBound} and Lemma~\ref{meansBoundLem}, we obtain:
\begin{align*}
\Sigma-I  &= (|E|/|S'|)M_E + O(\eps\log(1/\eps)) +O((|E|/|S'|)^2\|\mu^E-\mu^G\|_2^2) + \new{O(\eps^2\log(1/\eps)) } \\ 
                 &= (|E|/|S'|)M_E + O(\eps\log(1/\eps)) + O(|E|/|S'|)^2 \|M_E\|_2 \;,
\end{align*}
where the second line follows from the valid inequality  $\|M_E\|_2 \geq \|\mu^E-\mu^G\|_2^2$.
This completes the proof.
\end{proof}

\medskip

\noindent {\bf Case of Small Spectral Norm.} 
We are now ready to analyze the case that the mean vector $\mu^{S'}$ is returned by the algorithm in Step \ref{step:bal-small-G}.
In this case, we have that $\lambda^{\ast} \eqdef \|\Sigma-I\|_2 = O(\eps\log(1/\eps))$. Hence, Corollary \ref{MApproxCor-G} yields that
$$
(|E|/|S'|) \|M_E\|_2 \leq \lambda^{\ast} + O(\eps\log(1/\eps)) + O(|E|/|S'|)^2\|M_E\|_2 \;,
$$
which in turns implies that 
$$
(|E|/|S'|) \|M_E\|_2 = O(\eps\log(1/\eps)) \;.
$$
On the other hand, since $\|M_E\|_2 \geq \|\mu^E-\mu^G\|_2^2$, Lemma \ref{meansBoundLem} gives that
$$
\|\mu^{S'}-\mu^G\|_2 \leq (|E|/|S'|) \sqrt{\|M_E\|_2} + O(\eps\sqrt{\log(1/\eps)}) = O(\eps\sqrt{\log(1/\eps)}).
$$
\new{This proves part (i) of Proposition~\ref{prop:filter-gaussian-mean}.}

\medskip

\noindent {\bf Case of Large Spectral Norm.} 
We next show the correctness of the algorithm when it returns a filter in Step \ref{step:bal-large-G}.

We start by proving that if $\lambda^{\ast} \eqdef \|\Sigma-I\|_2  > C\eps\log(1/\eps)$, for a sufficiently large universal constant $C$,
then a value $T$ satisfying the condition in Step \ref{step:bal-large-G} exists. We first note that that $\|M_E\|_2$ is 
appropriately large. Indeed, by Corollary \ref{MApproxCor-G} and the assumption that $\lambda^{\ast}  > C\eps\log(1/\eps)$
we deduce that
\begin{equation} \label{eqn:large-me}
(|E|/|S'|) \|M_E\|_2  = \Omega (\lambda^{\ast}) \;.
\end{equation}
Moreover, using the inequality $\|M_E\|_2 \geq \|\mu^E-\mu^G\|_2^2$ and Lemma \ref{meansBoundLem} as above,
we get that
\begin{equation} \label{eqn:delta}
\|\mu^{S'}-\mu^G\|_2 \leq (|E|/|S'|) \sqrt{\|M_E\|_2} + O(\eps\sqrt{\log(1/\eps)}) \leq \delta/2 \;, 
\end{equation}
where we used the fact that $\delta \eqdef \sqrt{\eps \lambda^{\ast}} > C' \eps \sqrt{\log(1/\eps)}.$

Suppose for the sake of contradiction that for all $T>0$ we have that
\begin{equation*} 
\Pr_{X\in_u S'}\left[|v^{\ast} \cdot (X-\mu^{S'})|>T+\delta \right] \leq 8\exp(-T^2/2\nu)+8 \frac{\eps}{\new{T^2\log\left(d \log(\frac{d}{\eps\tau})\right)}} \;.
\end{equation*}
Using (\ref{eqn:delta}), we obtain that for all $T>0$ we have that
\begin{equation} \label{eqn:contradiction}
\Pr_{X\in_u S'}\left[|v^{\ast} \cdot (X-\mu^{G})|>T+\delta/2 \right] \leq 8\exp(-T^2/2\nu)+8 \frac{\eps}{\new{T^2\log\left(d \log(\frac{d}{\eps\tau})\right)}} \;.
\end{equation}
Since $E \subseteq S',$ for all $x \in \R^d$ we have that
$|S'|\Pr_{X\in_u S'}[X=x] \geq |E| \Pr_{Y\in_u E}[Y=x].$
This fact combined with (\ref{eqn:contradiction}) implies that for all $T>0$
\begin{equation} \label{eqn:contradiction2}
\Pr_{X\in_u E}\left[|v^{\ast} \cdot (X-\mu^{G})|>T+\delta/2 \right]  \ll (|S'|/|E|)\left(\exp(-T^2/2\nu)+ \frac{\eps}{\new{T^2\log\left(d \log(\frac{d}{\eps\tau})\right)}} \right) \;.
\end{equation}
We now have the following sequence of
inequalities:
\begin{align*}
\|M_E\|_2 &=  \E_{X\in_u E}\left[|v^{\ast}\cdot (X-\mu^G)|^2\right]  = 2 \int_{0}^{\infty} \Pr_{X\in_u E}\left[|v^{\ast}\cdot(X-\mu^G)|>T\right] T dT\\
& = 2 \int_{0}^{O(\sqrt{d \log(d/\eps\tau)})} \Pr_{X\in_u E}\left[|v^{\ast} \cdot(X-\mu^G)|>T\right] T dT \\
& \leq 2 \int_0^{O(\sqrt{d \log(d/\eps\tau)})} \min \left\{ 1,  \frac{|S'|}{|E|} \cdot \Pr_{X \in_u S'}\left[|v^{\ast} \cdot(X-\mu^{G})|>T\right]  \right\} TdT\\
& \ll \int_0^{4\sqrt{\nu \log(|S'|/|E|)}+\delta} T dT + (|S'|/|E|) \int_{4\sqrt{\nu \log(|S'|/|E|)} + \delta}^{O(\sqrt{d \log(d/\eps\tau)})}  
\Big( \exp(-T^2/2\nu)+\frac{\eps}{\new{T^2\log\left(d \log(\frac{d}{\eps\tau})\right)}} \Big)  T dT\\
& \ll  \log(|S'|/|E|) + \delta^2 + O(1) + \eps \cdot |S'|/|E| \\
& \ll  \log(|S'|/|E|) + \eps \lambda^{\ast} + \eps \cdot |S'|/|E| \;.
\end{align*}
Rearranging the above, we get that
$$(|E|/|S'|) \|M_E\|_2  \ll (|E|/|S'|)\log(|S'|/|E|) + (|E|/|S'|) \eps \lambda^{\ast} + \eps = O(\eps \log(1/\eps) + \eps^2 \lambda^{\ast}).$$
Combined with (\ref{eqn:large-me}), we obtain 
$\lambda^{\ast} = O(\eps \log(1/\eps))$, which is a contradiction if $C$ is sufficiently large.
Therefore, it must be the case that for some value of $T$ the condition in Step \ref{step:bal-large-G} is satisfied.

The following claim completes the proof:
\begin{claim} \label{claim:filter}
Fix $\alpha \eqdef \new{d\log(d/\eps\tau)\log(d \log(\frac{d}{\eps\tau}))}.$
We have that
$
\Delta(S,S'') \leq \Delta(S,S') - 2\eps/\new{\alpha} \;.
$
\end{claim}
\begin{proof}
Recall that $S' = (S\setminus L) \cup E,$ with $E$ and $L$ disjoint multisets such that $L \subset S.$
We can similarly write $S''=(S \setminus L') \cup E',$ with $L'\supseteq L$ and $E'\subset E.$
Since $$\Delta(S,S') - \Delta(S,S'')  = \frac{|E \setminus E'| - |L' \setminus L| }{|S|},$$
it suffices to show that $|E \setminus E'| \geq |L' \setminus L| + \eps|S|/\new{\alpha}.$
Note that $|L' \setminus L|$ is the number of points rejected by the filter that lie in $S \cap S'.$
Note that the fraction of elements of $S$
that are removed to produce $S''$ (i.e., satisfy $|v^{\ast}\cdot(x-\mu^{S'})|>T+\delta$) is at most $2\exp(-T^2/2\nu) + \eps/\new{\alpha}$.
\new{This follows from Claim~\ref{claim:conc} and the fact that $T  = O(\sqrt{ d\log(d/\eps\tau)})$.}

Hence, it holds that $|L' \setminus L| \leq (2\exp(-T^2/2\nu) + \eps/\new{\alpha}) |S|.$
On the other hand, Step~\ref{step:bal-large-G} of the algorithm ensures that the fraction of elements of $S'$ that are rejected
by the filter is at least $8\exp(-T^2/2\nu)+8\eps/\new{\alpha})$. Note that
$|E \setminus E'|$ is the number of points rejected by the filter that lie in $S' \setminus S.$
Therefore, we can write:
\begin{align*}
|E\setminus E'| & \geq (8\exp(-T^2/2\nu)+8\eps/\new{\alpha})|S'| - (2\exp(-T^2/2\nu) + \eps/\new{\alpha})  |S| \\
				& \geq (8\exp(-T^2/2\nu)+8\eps/\new{\alpha})|S|/2 - (2\exp(-T^2/2\nu) + \eps/\new{\alpha}) |S| \\
				& \geq  (2\exp(-T^2/2\nu) + 3\eps/\new{\alpha}) |S| \\
				& \geq |L' \setminus L| + 2\eps|S|/\new{\alpha}  \;,
\end{align*}
where the second line uses the fact that $|S'| \ge |S|/2$
and the last line uses the fact that $|L' \setminus L| / |S| \leq 2\exp(-T^2/2\nu) + \eps/\new{\alpha}.$
Noting that $\log(d/\eps\tau) \geq 1$, this completes the proof of the claim.
\end{proof}


\subsubsection{Proof of Lemma~\ref{lem:random-good-gaussian-mean}}
\begin{proof}
Let $N = \Omega( (d/\eps^2) \poly\log(d/\eps\tau))$ be the number of samples drawn from $G$.
For (i), the probability that a coordinate of a sample is at least $\sqrt{2\nu\log(Nd/3\tau)}$ 
is at most $\tau/3dN$ by Fact \ref{fact:tail-bound}. By a union bound, 
the probability that all coordinates of all samples are smaller than $\sqrt{2\nu\log(Nd/3\tau)}$ 
is at least $1-\tau/3$. In this case, $\|x\|_2 \leq \sqrt{ 2\nu d \log(Nd/3\tau)} = O(\sqrt{d \nu \log(N\nu/\tau)})$.

After translating by $\mu^G$, we note that (iii) follows 
immediately from Lemmas 4.3 of  \cite{DKKLMS} and (iv) follows from Theorem 5.50 of \cite{Vershynin}, 
as long as $N =\Omega(\nu^4 d\log(1/\tau)/\eps^2)$, with probability at least $1-\tau/3$. 
It remains to show that, conditioned on (i), (ii) holds with probability at least $1-\tau/3$.

To simplify some expressions, let $\delta := \eps/(\log(d \log d/\eps\tau))$ and $R=C\sqrt{d\log(|S|/\tau)}$.
We need to show that for all unit vectors $v$ and all $0 \leq T \leq R$ that
\begin{equation} \label{eq:grail}
\left| \Pr_{X\in_u S}[|v \cdot (X-\mu^G)| > T] - \Pr_{X\sim G}[|v \cdot (X-\mu^G) > T \ge 0] \right| \leq \frac{\delta}{T^2} \;.
\end{equation}

Firstly, we show that for all unit vectors $v$ and $T >0$
$$
\left| \Pr_{X\in_u S}[|v \cdot (X-\mu^G)| > T] - \Pr_{X\sim G}[|v \cdot (X-\mu^G)| > T \ge 0] \right| \leq \frac{\delta}{10 \nu \ln(1/\delta)}
$$
with probability at least $1-\tau/6$. Since the VC-dimension of the set of all halfspaces is $d+1$, 
this follows from the VC inequality~\cite{DL:01}, 
since we have more than $\Omega(d/(\delta/(10 \nu \log(1/\delta))^2)$ samples. 
We thus only need to consider the case when $T \geq \sqrt{10 \nu \ln(1/\delta)}$.

\begin{lemma} For any fixed unit vector $v$ and $T > \sqrt{10 \nu\ln(1/\delta)}$, 
except with probability $\exp(-N\delta/(6C\nu))$, we have that
$$\Pr_{X\in_u S}[|v \cdot (X-\mu^G)| > T] \leq \frac{\delta}{ C T^2} \;,$$
where $C=8$.
\end{lemma}
\begin{proof}
Let $E$ be the event that $|v \cdot (X-\mu^G)| > T$. Since $G$ is sub-gaussian, 
Fact \ref{fact:tail-bound} yields that $\Pr_G[E]= \Pr_{Y \sim G}[|v \cdot (X-\mu^G)| > T] \leq \exp(-T^2/(2\nu))$. 
Note that, thanks to our assumption on $T$, we have that $T \leq \exp(T^2/(4\nu))/2C$, 
and therefore $T^2\Pr_G[E] \leq \exp(-T^2/(4\nu))/2C \leq \delta/2C$.

Consider $\E_S[\exp(t^2/ (3\nu) \cdot  N \Pr_S[E])]$.
Each individual sample $X_i$ for $1 \leq i \leq N$, is an independent copy of $Y \sim G$, and hence:
\begin{align*}
\E_S\left[\exp \left( \frac{T^2}{3\nu} \cdot  N \Pr_{S}[E] \right) \right] 
& = \E_S \left[ \exp \left( \frac{T^2}{3\nu} \right) \cdot \sum_{i=1}^n 1_{X_i \in E}) \right] \\
&= \prod_{i=1}^N \E_{X_i}\left[ \exp \left( \frac{T^2}{3\nu} \right) \cdot \sum_{i=1}^n 1_{X_i \in E})\right] \\
& = \left(\exp \left( \frac{T^2}{3\nu} \right) \Pr_G [G] + 1 \right)^N \\
&\stackrel{(a)}{\leq}  \left( \exp \left( \frac{T^2}{6 \nu} \right) + 1 \right)^N \\
&\stackrel{(b)}{\leq} (1 + \delta^{5/3})^N \\
&\stackrel{(c)}{\leq} \exp (N \delta^{5/3}) \; ,
\end{align*}
where (a) follows from sub-gaussianity, (b) follows from our choice of $T$, and (c) comes from the fact that $1 + x \leq e^x$ for all $x$.

Hence, by Markov's inequality, we have
\begin{align*}
\Pr \left[ \Pr_S [E] \geq \frac{\delta}{C T^2} \right] &\leq \exp \left( N \delta^{5/3} - \frac{\delta N}{3C} \right) \\
&= \exp (N \delta (\delta^{2/3} - 1 / (3C))) \; .
\end{align*}
Thus, if $\delta$ is a sufficiently small constant and $C$ is sufficiently large, this yields the desired bound.
\end{proof}

Now let $\mathcal{C}$ be a $1/2$-cover in Euclidean distance 
for the set of unit vectors of size $2^{O(d)}$. 
By a union bound, for all $v' \in \mathcal{C}$ and $T'$ a power of 2 between $\sqrt{4\nu\ln(1/\delta)}$ and $R$, we have that
$$\Pr_{X\in_u S}[|v' \cdot (X-\mu^G)| > T'] \leq \frac{\delta}{ 8 T^2} $$
except with probability 
$$2^{O(d)} \log(R)\exp(-N\delta/6C\nu) = \exp\left(O(d) + \log \log R -N\delta/6C\nu\right) \leq \tau/6 \;.$$
However, for any unit vector $v$ and $\sqrt{4\nu\ln(1/\delta)} \leq T \leq R$, there is a $v' \in \mathcal{C}$ 
and such a $T'$ such that for all $x \in \R^d$, we have $|v \cdot (X-\mu^G)| \geq |v' \cdot (X-\mu^G)|/2$, 
and so $|v' \cdot (X-\mu^G)| > 2T'$ implies  $|v' \cdot (X-\mu^G)| > T.$

Then, by a union bound, (\ref{eq:grail}) holds simultaneously for all unit vectors $v$ and all $0 \leq T \leq R$,  
with probability a least $1-\tau/3$. This completes the proof.
\end{proof}

\subsection{Robust Mean Estimation Under Second Moment Assumptions} \label{ssec:filter-2mom}

In this section, we use our filtering technique to give a near sample-optimal computationally efficient algorithm 
to robustly estimate the mean of a density with a second moment assumption. We show:

\begin{theorem} \label{thm:second-moment-app}
Let $P$ be a distribution on $\R^d$ with unknown mean vector $\mu^{P}$ and unknown covariance matrix 
$\Sigma_P \preceq I$. Let $S$ be an $\eps$-corrupted set of samples 
from $P$ of size $\Theta((d/\eps) \log d)$. 
Then there exists an algorithm that given $S$, with probability $2/3$, 
outputs $\wh{\mu}$ with $\|\wh{\mu} - \mu^{P}\|_2 \leq O(\sqrt{\eps})$  in time $\poly(d/\eps)$.
\end{theorem}

Note that Theorem~\ref{thm:second-moment} follows straightforwardly from the above (divide 
every sample by $\sigma$, run the algorithm of Theorem~\ref{thm:second-moment-app}, and multiply
its output by $\sigma$).

As usual in our filtering framework, the algorithm will iteratively look at the top eigenvalue and eigenvector 
of the sample covariance matrix and return the sample mean if this eigenvalue is small (Algorithm \ref{alg:filter-2nd-moment}). 
The main difference between this and the filter algorithm for the sub-gaussian case is how we choose the threshold for the filter. 
Instead of looking for a violation of a concentration inequality, here we will choose a threshold {\em at random} (with a bias towards higher thresholds).
The reason is that, in this setting, the variance in the direction we look for a filter in needs to be a constant multiple larger -- 
instead of the typical $\tilde{\Omega}(\eps)$ relative for the sub-gaussian case. Therefore, randomly choosing a threshold weighted 
towards higher thresholds suffices to throw out more corrupted samples than uncorrupted samples {\em in expectation}. 
Although it is possible to reject many good samples this way, the algorithm still only rejects a total of $O(\eps)$ samples with high probability.

We would like our good set of samples to have mean close to that of $P$ and bounded variance in all directions.
This motivates the following definition:

\begin{definition} \label{def:good-2nd-moment} 
We call a set $S$ $\eps$-good for a distribution $P$ with mean $\mu^{P}$ and covariance $\Sigma_P \preceq I$
if the mean $\mu^S$ and covariance $\Sigma^S$ of $S$ satisfy 
$\|\mu^S-\mu^{P}\|_2 \leq \sqrt{\eps}$ and $\|\Sigma^S\|_2 \leq 2$.
\end{definition}

However, since we have no assumptions about higher moments, 
it may be be possible for outliers to affect our sample covariance too much. 
Fortunately, such outliers have small probability and do not contribute too much to the mean, 
so we will later reclassify them as errors. 

\begin{lemma} \label{lem:samples-good} 
Let $S$ be $N=\Theta((d/\eps)\log d)$ samples drawn from $P$. 
Then, with probability at least $9/10$, a random $X \in_u S$ satisfies
\begin{itemize}
\item[(i)] $\|\E_S [X] -\mu^{P}\|_2 \leq \sqrt{\eps}/3$,
\item[(ii)] $\Pr_S \left[\|X-\mu^{P}\|_2 \geq 80\sqrt{d/\eps}\right] \leq \eps/160$,
\item[(iii)] $\left \|  \E_S \left[ (X - \mu^P) \cdot 1_{\|X-\mu^{P} \|_2 \leq 80\sqrt{d/\eps}} \right] \right\|_2 \leq \sqrt{\eps}/3$, and
\item[(iv)] $\left\| \E_S \left[(X-\mu^{P}) (X-\mu^{P})^T \cdot 1_{\|X-\mu^{P}\|_2 \leq 80\sqrt{d/\eps}}  \right] \right\|_2 \leq 3/2$.
\end{itemize}
\end{lemma}
\begin{proof}
For (i), note that 
$$\E_S[\|\E[X] -\mu^{P}\|_2^2]=\sum_i \E_S[(\E[X]_i -\mu^{P}_i)^2] \leq d/N \leq \eps/360 \;,$$
and so by Markov's inequality, with probability at least $39/40$, 
we have $\| \E[X] -\mu^{P}\|_2^2 \leq \eps/9$.

For (ii), similarly to (i), note that  
$$\E[\|Y -\mu^{P}\|_2^2]=\sum_i \E \left[ (Y_i -\mu^{P}_i)^2 \right] \leq d \;,$$ 
for $Y \sim P$.  
By Markov's inequality, $\Pr[\|Y-\mu^{P}\|_2 \geq 80\sqrt{d/\eps}] \leq \eps/160$ with probability at least $39/40$.

For (iii), let $\nu = \E_{X \sim P} [X \cdot 1_{\| X - \mu^P \|_2 \leq 80 \sqrt{d / \eps}}]$ be the true mean of the distribution when we condition on the event that $\| X - \mu^P \|_2 \leq 80 \sqrt{d / \eps}$.
By the same argument as (i), we know that 
\[
\left\| \E_{X \in_u S} \left[X \cdot 1_{\|X-\mu^{P} \|_2 \leq 80\sqrt{d/\eps}}\right] - \nu \right\|_2 \leq \sqrt{\eps} / 9 \; , \]
with probability at least $39 / 40$.
Thus it suffices to show that $\left \| \nu - \mu^P \cdot 1_{\|X-\mu^{P} \|_2 \leq 80\sqrt{d/\eps}} \right\|_2 \leq \sqrt{\eps} / 10$.
To do so, it suffices to show that for all unit vectors $v \in \R^d$, we have 
\[
\left| \left\langle v, \nu - \mu^P \cdot 1_{\|X-\mu^{P} \|_2 \leq 80\sqrt{d/\eps}} \right\rangle \right| < \sqrt{\eps} / 10 \; .
\]
Observe that for any such $v$, we have
\begin{align*}
\left\langle v, \mu^P \cdot 1_{\|X-\mu^{P} \|_2 \leq 80\sqrt{d/\eps}} - \nu \right\rangle &=  \E_{X \sim P} \left[  \left\langle v, X - \mu^P \right\rangle \cdot 1_{\|X-\mu^{P} \|_2 \leq 80\sqrt{d/\eps}}  \right] \\
&\stackrel{(a)}{\leq} \sqrt{ \E_{X \sim P} \left[ \langle v, X - \mu^P \rangle^2 \right] \Pr_{X \sim P} [\| X - \mu^P \|_2 \geq 80 \sqrt{d / \eps}] } \\
&\stackrel{(b)}{=} \sqrt{v^T \Sigma_P v \cdot \Pr_{X \sim P} \left[ \| X - \mu^P \|_2 \geq 80 \sqrt{d / \eps} \right]} \\
&\stackrel{(c)}{\leq} \sqrt{\eps} / 10 \; ,
\end{align*}
where (a) follows from Cauchy-Schwarz, and (b) follows from the definition of the covariance, and (c) follows from the assumption that $\Sigma_P \preceq I$ and from Markov's inequality.
 
For (iv), we require the following Matrix Chernoff bound:
 
\begin{lemma}[Part of Theorem 5.1.1 of \cite{Tropp2015}]
Consider a sequence of $d \times d$ positive semi-definite random matrices $X_k$ with $\|X_k\|_2 \leq L$ for all $k$. 
Let $\mu^{\max} = \left\| \sum_k \E[X_k] \right\|_2$. Then, for $\theta > 0$,
$$\E\left[\left\|\sum_k X_k \right\|_2 \right] \leq (e^\theta -1) \mu^{\max}/\theta + L \log(d)/\theta \;,$$
and for any $\delta > 0$,
$$\Pr\left[\left\|\sum_k X_k \right\|_2 \geq (1+\delta)\mu^{\max}\right] \leq d (e^\delta/(1+\delta)^{1+\delta})^{\mu^{\max}/L} \; .$$
\end{lemma}
 
\medskip
 
\noindent We apply this lemma with $X_k =(x_k-\mu^{P}) (x_k-\mu^{P})^T 1_{\|x_k-\mu^{P}\|_2 \leq 80\sqrt{d/\eps}} $ for $\{x_1,\dots,x_N\} = S$. 
Note that $\|X_k\|_2 \leq (80)^2d/\eps = L$ and that $\mu^{\max} \leq N \|\Sigma_{P}\|_2 \leq N$. 
 
Suppose that $\mu^{\max} \leq N/80$. 
Then, taking $\theta=1$, we have 
$$\E[\left\|\sum_k X_k \right\|_2 ] \leq (e-1)N/80 + O(d\log(d)/\eps) \;.$$ 
By Markov's inequality, except with probability $39/40$, 
we have $\|\sum_k X_k \|_2 \leq N + O(d\log(d)/\eps) \leq 3N/2$, 
for $N$ a sufficiently high multiple of $d\log(d)/\eps$.
 
Suppose that $\mu^{\max} \geq N/80$, 
then we take $\delta=1/2$ and obtain 
$$\Pr\left[\left\|\sum_k X_k \right\|_2 \geq 3\mu^{\max}2\right] \leq d (e^{3/2}/(5/2)^{3/2})^{N\eps/20d} \;.$$ 
For $N$ a sufficiently high multiple of $d\log(d)/\eps$, 
we get that $\Pr[\left\|\sum_k X_k \right\|_2 \geq 3\mu^{\max}/2] \leq 1/40$. 
Since $\mu^{\max} \leq N$, we have with probability at least $39/40$, $\left\|\sum_k X_k \right\|_2 \leq 3N/2$.
 
\smallskip 

Noting that $\left\|\sum_k X_k \right\|_2/N = \|\E[1_{\|X-\mu^{P}\|_2 \leq 80\sqrt{d/\eps}} (X-\mu^{P}) (X-\mu^{P})^T ]\|_2$, we obtain (iv).
By a union bound, (i)-(iv) all hold simultaneously with probability at least $9/10$.
 \end{proof}

Now we can get a $2\eps$-corrupted good set from
an $\eps$-corrupted set of samples satisfying 
Lemma \ref{lem:samples-good}, by reclassifying outliers as errors:
\begin{lemma} \label{lem:naive-prune} 
Let $S= R \cup E \setminus L$, where $R$ 
is a set of  $N=\Theta(d \log d /\eps)$ samples drawn from $P$ 
and $E$ and $L$ are disjoint sets with $|E|,|L| \leq  \eps$. 
Then, with probability $9/10$,  we can also write $S=G \cup E' \setminus L'$, 
where $G \subseteq R$ is $\eps$-good, 
$L' \subseteq L$ and $E' \subseteq E'$ has $|E'| \leq 2\eps |S|$. 
\end{lemma}
\begin{proof}
Let $G=\{x \in R: \|x\|_2 \leq 80\sqrt{d/\eps} \}$. 
Condition on the event that $R$ satisfies Lemma \ref{lem:samples-good}.
By Lemma \ref{lem:samples-good}, this occurs with probability at least $9 / 10$.

Since $R$ satisfies (ii) of Lemma \ref{lem:samples-good}, 
$|G|-|R| \leq \eps|R|/160 \leq \eps|S|$. 
Thus, $E'= E \cup (R \setminus G)$ has $|E'| \leq 3\eps/2$. 
Note that (iv) of Lemma \ref{lem:samples-good} for $R$ in terms of $G$ 
is exactly $|G|\|\Sigma^G\|_2/|R| \leq 3/2$, and so $ \|\Sigma^G\|_2 \leq 3|R|/(2|G|) \leq 2$.

It remains to check that $\|\mu^G - \mu^{P}\|_2 \leq \sqrt{\eps}$. 
We have
\begin{align*}
\left\| |G| \cdot \mu^G - |G| \cdot \mu^P \right\|_2   &= |R| \cdot \left\|\E_{X \sim_u R} \left[ (X - \mu^P) \cdot 1_{\|X-\mu^{P} \|_2 \leq 80\sqrt{d/\eps}} \right]  \right\|_2\\
&\leq |R| \cdot \sqrt{\eps} / 3 \; ,
\end{align*}
where the last line follows from (iii) of Lemma \ref{lem:samples-good}.
Since we argued above that $|R| / |G| \geq 2/3$, dividing this expression by $|G|$ yields the desired claim.

\end{proof}

\begin{algorithm}[htb]
\begin{algorithmic}[1]
\Function{FilterUnder2ndMoment}{$S$}
\State Compute $\mu^S$, $\Sigma^S$, the mean and covariance matrix of $S$.
\State Find the eigenvector $v^{\ast}$ with highest eigenvalue $\lambda^{\ast}$ of $\Sigma^S$.
\If{$\lambda^{\ast}  \leq 9$}
	\State \textbf{return} $\mu^S$ \label{step:return-mean}
\Else
	\State Draw $Z$ from the distribution on $[0,1]$ with probability density function $2x$. 
	\State Let $T= Z \max\{ |v^{\ast} \cdot x - \mu^S| : x \in S \}.$
	\State Return the set $S'=\{x \in S :|v^{\ast} \cdot (X -\mu^S)| < T \}$.
\EndIf
\EndFunction
\end{algorithmic}
\caption{Filter under second moment assumptions}
\label{alg:filter-2nd-moment}
\end{algorithm}

An iteration of {\textsc FilterUnder2ndMoment} may throw out more samples from $G$ than corrupted samples. 
However, in expectation, we throw out many more corrupted samples than from the good set:

\begin{proposition} \label{prop:iteration-2nd-moment}
If we run {\textsc FilterUnder2ndMoment} on a set $S=G \cup E \setminus L$ 
for some $\eps$-good set $G$ and disjoint $E, L$ with $|E| \leq 2\eps |S|,  |L| \leq 9\eps|S|$, 
then either it returns $\mu^S$ with $\|\mu^S-\mu^{P}\|_2 \leq O(\sqrt{\eps})$, 
or else it returns a set $S' \subset S$ with $S'=G \cup E' \setminus L'$ for disjoint $E'$ and $L'$. 
In the latter case we have $\E_Z[|E'| + 2|L'|] \leq |E| + 2|L|$.
\end{proposition}

For $D \in \{G,E,L,S\}$, let $\mu^D$ be the mean of $D$ and $M_D$ be the matrix $\E_{X \in_u D}[(X-\mu^S)(X-\mu^S)^T]$.

\begin{lemma} \label{lem:MG} 
If $G$ is an $\eps$-good set with $x \leq 40 \sqrt{d/\eps}$ for $x \in S \cup G$, then
$\|M_G\|_2 \leq 2\|\mu^G-\mu^S\|_2^2 + 2\; .$
\end{lemma}
\begin{proof}
For any unit vector $v$, we have
\begin{align*}
v^T M_G v  &= \E_{X \in_u G}[(v \cdot (X-\mu^S))^2] \\
& = \E_{X \in_u G}[(v \cdot (X-\mu^G) + v \cdot (\mu^{P}-\mu^G)) ^2] \\
& = v^T \Sigma^G v +  (v \cdot (\mu^{G}-\mu^S))^2\\
& \leq 2 + 2 \|\mu^G-\mu^S\|_2^2 \;.
\end{align*}
\end{proof}

\begin{lemma} \label{lem:ML} 
We have that
$|L|\|M_L\|_2 \leq 2|G|(1+\|\mu^G-\mu^S\|_2^2) \; .$
\end{lemma}
\begin{proof}
Since $L \subseteq G$, for any unit vector $v$, we have
\begin{align*}
|L| v^T M_L v  &= |L| \E_{X \in_u L}[(v \cdot (X-\mu^S))^2] \\
& \leq |G| \E_{X \in_u G}[(v \cdot (X-\mu^S))^2] \\
& \leq 2|G|(1+\|\mu^G-\mu^S\|_2^2) \;.
\end{align*}
\end{proof}

\begin{lemma} \label{lem:main-distance-bound} 
$\|\mu^G-\mu^S\|_2 \leq \sqrt{2\eps \|M_S\|_2} + 12\sqrt{\eps}$.
\end{lemma}
\begin{proof}
We have that $|E|M_E \leq |S|M_S + |L|M_L$ 
and so $$|E|\|M_E\|_2 \leq |S|\|M_S\|_2 + 2|G|(1+\|\mu^G-\mu^S\|_2^2) \;.$$ 
By Cauchy Schwarz, we have that $\|M_E\|_2 \geq \|\mu^E - \mu^S\|_2^2$, and so
$$ \sqrt{|E|} \|\mu^E - \mu^S\|_2 \leq \sqrt{|S|\|M_S\|_2 + 2|G|(1+\|\mu^G-\mu^S\|_2^2)} \; .$$
By Cauchy-Schwarz and Lemma \ref{lem:ML}, we have that
$$\sqrt{|L|}  \|\mu^L - \mu^S\|_2 \leq \sqrt{|L| \|M_L\|_2} \leq \sqrt{2|G|(1+\|\mu^G-\mu^S\|_2^2)} \;.$$
Since $|S|\mu^S = |G|\mu^G +|E|\mu^E - |L| \mu^L$
and $|S|=|G|+|E|-|L|$, we get 
$$|G|(\mu^G-\mu^S)=|E|(\mu^E - \mu^S) - |L|(\mu^E - \mu^S) \;.$$ 
Substituting into this, we obtain
$$|G| \| \mu^G-\mu^S \|_2 \leq \sqrt{|E||S|\|M_S\|_2 + 2|E||G|(1+\|\mu^G-\mu^S\|_2^2)} + \sqrt{2|L||G|(1+\|\mu^G-\mu^S\|_2^2)} \;.$$
Since for $x,y >0$, $\sqrt{x+y}\leq \sqrt{x}+\sqrt{y}$, we have
$$|G| \| \mu^G-\mu^S \|_2 \leq \sqrt{|E||S|\|M_S\|_2} +(\sqrt{2|E||G|}+\sqrt{2|L||G|})(1+ \|\mu^G-\mu^S\|_2) \;.$$
Since $||G| - |S| | \leq \eps|S|$ and $|E| \leq 2\eps|S|,|L| \leq 9\eps|S|$, we have
 $$\|\mu^G-\mu^S\|_2 \leq \sqrt{2\eps \|M_S\|_2} +(6\sqrt{\eps})(1+ \|\mu^G-\mu^S\|_2) \;.$$
Moving the $\|\mu^G-\mu^S\|_2$ terms to the LHS, using $6\sqrt{\eps} \leq 1/2$, gives
 $$\|\mu^G-\mu^S\|_2 \leq \sqrt{2\eps \|M_S\|_2} +12\sqrt{\eps} \;.$$
\end{proof}
 
Since $\lambda^{\ast}=\|M_S\|_2$, the correctness if we return the empirical mean is immediate.
\begin{corollary} \label{cor:correct-mean} 
If $\lambda^{\ast}  \leq 9$,  we have that 
$\|\mu^G-\mu^S\|_2= O(\sqrt{\eps})$.
\end{corollary}

From now on, we assume $\lambda^{\ast} > 9$. 
In this case we have $\|\mu^G-\mu^S\|_2^2 \leq O(\eps \lambda^{\ast})$.
Using Lemma \ref{lem:MG}, we have
$$\|M_G\|_2 \leq 2 + O(\eps \lambda^{\ast}) \leq 2 + \lambda^{\ast}/5$$
for sufficiently small $\eps$. Thus, we have that

\begin{equation}\label{eq:constant-times-variance}
v^{\ast T} M_S v^{\ast} \geq 4 v^{\ast T} M_G v^{\ast} \;.
\end{equation}

Now we can show that in expectation, 
we throw out many more corrupted points from $E$ than from $G \setminus L$:
\begin{lemma} \label{lem:expected-good} 
Let $S'=G \cup E' \setminus L'$ for disjoint $E', L'$ be the set of samples returned by the iteration. 
Then we have
$\E_Z[|E'| + 2|L'|] \leq |E| + 2|L|$.
\end{lemma}
\begin{proof}
Let $a= \max_{x \in S} |v^{\ast} \cdot x - \mu^S|$. 
Firstly, we look at the expected number of samples we reject:
\begin{align*}
\E_Z[|S'|]-|S| & = \E_Z\left[|S|\Pr_{X \in_u S}[|X-\mu^S| \geq a Z]\right] \\
& = |S| \int_{0}^{1} \Pr_{X \in_u S}\left[|v^{\ast} \cdot(X-\mu^S)| \geq a x \right] 2x dx \\
& = |S| \int_{0}^{a} \Pr_{X \in_u S}\left[|v^{\ast} \cdot(X-\mu^S)| \geq T \right] (2T /a)  dT \\
& = |S| \E_{X \in_u S}\left[(v^{\ast} \cdot(X-\mu^S))^2\right]/a \\
&=  (|S|/a) \cdot v^{\ast T} M_S v^{\ast} \;.
\end{align*}
Next, we look at the expected number of false positive samples we reject, 
i.e., those in $L' \setminus L$.
\begin{align*}
\E_Z[|L'|]-|L| & = \E_Z\left[(|G|-|L|)\Pr_{X \in_u G \setminus L}\left[|X-\mu^S| \geq T\right] \right] \\
& \leq \E_Z\left[|G|\Pr_{X \in_u G}[|v^{\ast} \cdot(X-\mu^S)| \geq a Z] \right] \\
& = |G| \int_{0}^{1} \Pr_{X \in_u G}[|v^{\ast} \cdot(X-\mu^S)| \geq a x] 2x \; dx \\
& = |G| \int_{0}^{a} \Pr_{X \in_u G}[|v^{\ast} \cdot(X-\mu^S)| \geq T ] (2T /a) \; dT \\
& \leq |G| \int_{0}^{\infty} \Pr_{X \in_u G}[|v^{\ast} \cdot(X-\mu^S)|\geq T ] (2T /a) \; dT  \\
& = |G| \E_{X \in_u G} \left[(v^{\ast} \cdot(X-\mu^S))^2 \right]/a \\ 
&= (|G|/a) \cdot v^{\ast T} M_G v^{\ast} \;.
\end{align*}
Using (\ref{eq:constant-times-variance}), we have $|S| v^{\ast T} M_S v^{\ast} \geq 4 |G| v^{\ast T} M_G v^{\ast}$ 
and so $\E_Z[S']-S \geq 3(\E_Z[L']-L)$. Now consider that 
$|S'|=|G|+|E'|-|L'|=|S| - |E|+|E'|+|L|-|L'|$, and thus 
$|S'|-|S|=|E|-|E'| + |L'|-|L|$. 
This yields that $|E|- \E_Z[|E'|] \geq 2(\E_Z[L']-L)$, 
which can be rearranged to $\E_Z[|E'| + 2|L'|] \leq |E| + 2|L|$.
\end{proof}

\begin{proof}[Proof of Proposition \ref{prop:iteration-2nd-moment}] 
If $\lambda^{\ast} \leq 9$, then we return the mean in Step \ref{step:return-mean}, 
and by Corollary \ref{cor:correct-mean}, $\|\mu^S-\mu^{P}\|_2 \leq O(\sqrt{\eps})$. 

If $\lambda^{\ast} > 9$, then we return $S'$. Since at least one element of $S$ has 
$|v^{\ast} \cdot X| =\max_{x \in S} |v^{\ast} \cdot X|$, 
whatever value of $Z$ is drawn, 
we still remove at least one element, 
and so have $S' \subset S$. 
By Lemma \ref{lem:expected-good}, 
we have $\E_Z[|E'| + 2|L'|] \leq |E| + 2|L|$.
\end{proof}

\begin{proof}[Proof of Theorem \ref{thm:second-moment-app}] 
Our input is a set $S$ of $N=\Theta((d/\eps) \log d)$
$\eps$-corrupted samples so that with probability $9/10$, 
$S$ is a $2\eps$-corrupted set of $\eps$-good samples for $P$ 
by Lemmas \ref{lem:samples-good} and \ref{lem:naive-prune}. 
We have a set $S=G \cup E' \setminus L$, where $G'$ is an $\eps$-good set, 
$|E| \leq 2\eps$, and $|L| \leq \eps$.
Then, we iteratively apply \textsc{FilterUnder2ndMoment} 
until it outputs an approximation to the mean. 
Since each iteration removes a sample, this must happen within $N$ iterations. 
The algorithm takes at most $\poly(N,d)=\poly(d,1/\eps)$ time.

As long as we can show that the conditions of Proposition \ref{prop:iteration-2nd-moment} hold in each iteration, 
it ensures that  $\|\mu^S-\mu^{P}\|_2 \leq O(\sqrt{\eps})$.
However, the condition that $|L| \leq 9\eps|S|$ 
need not hold in general. Although in expectation we reject many more samples in $E$ than $G$, 
it is possible that we are unlucky and reject many samples in $G$, 
which could make $L$ large in the next iteration. 
Thus, we need a bound on the probability that we ever have $|L| > 9 \eps$.

We analyze the following procedure: We iteratively run \textsc{FilterUnder2ndMoment} 
starting with a set $S_i \cup E_i \setminus L_i$ of samples with $S_0=S$ 
and producing a set $S_{i+1}=G \cup E_{i+1} \setminus L_{i+1}$. 
We stop if we output an approximation to the mean or if $|L_{i+1}| \geq 13 \eps |S|$. 
Since we do now always satisfy the conditions of Proposition \ref{prop:iteration-2nd-moment}, this 
gives that $\E_Z[|E_{i+1}| + |L_{i+1}|]=|E_i|+2|L_i|$.
This expectation is conditioned on the state of the algorithm after previous iterations, which is determined by $S_i$.
Thus, if we consider the random variables $X_i=|E_i|+2|L_i|$, then we have $\E[X_{i+1} | S_i] \leq X_i$, 
i.e., the sequence $X_i$ is a sub-martingale with respect to $X_i$. 
Using the convention that $S_{i+1}=S_i$, if we stop in less than $i$ iterations, 
and recalling that we always stop in $N$ iterations, the algorithm fails if and only if $|L_N| > 9 \eps |S|$. 
By a simple induction or standard results on sub-martingales, we have 
$\E[X_N] \leq X_0$. Now $X_0 = |E_0|+2|L_0| \leq 3 \eps |S|$. 
Thus, $\E[X_N] \leq 3\eps|S|$. By Markov's inequality, except with probability $1/6$, 
we have $X_N \leq 18 \eps |S|$. In this case, 
$|L_N| \leq X_N/2 \leq 9 \eps|S|$. 
Therefore, the probability that we ever have $|L_i| > 9 \eps$ is at most $1/6$.

By a union bound, the probability that the uncorrupted samples satisfy Lemma \ref{lem:samples-good} 
and Proposition \ref{prop:iteration-2nd-moment} applies to every iteration
is at least $9/10-1/6 \geq 2/3$. 
Thus, with at least $2/3$ probability, the algorithm outputs a vector $\wh{\mu}$ with $\|\wh{\mu}-\mu^{P}\|_2 \leq O(\sqrt{\eps})$.
\end{proof} 

\subsection{Robust Covariance Estimation} \label{ssec:cov}
In this subsection, we give a near sample-optimal efficient robust estimator for the covariance
of a zero-mean Gaussian density, thus proving Theorem~\ref{unknownCovarianceTheorem}.
Our algorithm is essentially identical to the filtering algorithm given in Section~8.2 of~\cite{DKKLMS}.
As in Section~\ref{sec:filter-subgaussian} the only difference is a weaker definition of the ``good set of samples'' (Definition~\ref{def:good-sample-cov}) 
and a concentration argument (Lemma~ \ref{lem:random-good-gaussian-cov}) showing that a random set of uncorrupted 
samples of the appropriate size is good with high probability. Given these, the analysis of this subsection follows straightforwardly
from the analysis in Section~8.2 of~\cite{DKKLMS} by plugging in the modified parameters.

\medskip

The algorithm \textsc{Filter-Gaussian-Unknown-Covariance} to robustly estimate the covariance of a mean $0$ Gaussian in \cite{DKKLMS}
is as follows:

\begin{algorithm}
\begin{algorithmic}[1]
\Procedure{Filter-Gaussian-Unknown-Covariance}{$S',\eps,\tau$}
\INPUT A multiset $S'$ such that there exists an $(\eps,\tau)$-good set $S$ with $\Delta(S, S') \le 2\eps$
\OUTPUT Either a set $S''$ with $\Delta(S,S'') < \Delta(S,S')$ or the parameters of a Gaussian $G'$ with $d_{TV}(G,G') = O(\epsilon\log(1/\epsilon)).$

Let $C>0$ be a sufficiently large universal constant.

\State Let $\Sigma'$ be the matrix $\E_{X\in_u S'}[XX^T]$ and let $G'$ be the mean $0$ Gaussian with covariance matrix $\Sigma'$.
\If {there is any $x\in S'$ so that $x^T(\Sigma')^{-1} x \geq Cd\log(|S'|/\tau)$}
\State \textbf{return} $S''=S'-\{x:x^T(\Sigma')^{-1} x \geq Cd\log(|S'|/\tau)\}$.\label{removeOutlierStep}
\EndIf

\State Compute an approximate eigendecomposition of $\Sigma'$ and use it to compute $\Sigma'^{-1/2}$
\State Let $x_{(1)}, \ldots, x_{(|S'|)}$ be the elements of $S'$.
\State For $i = 1, \ldots, |S'|$, let $y_{(i)} = \Sigma'^{-1/2} x_{(i)}$ and $z_{(i)} = y_{(i)}^{\otimes 2}$.
\State Let $T_{S'}= - I^\flat I^{\flat T} + (1/|S'|)\sum_{i=1}^{|S'|} z_{(i)} z_{(i)}^T.$ 
\State Approximate the top eigenvalue $\lambda^{\ast}$ and corresponding unit eigenvector $v^{\ast}$ of $T_{S'}.$.
\State Let $p^{\ast}(x) = \frac{1}{\sqrt{2}} ((\Sigma'^{-1/2} x)^T v^{\ast \sharp} (\Sigma'^{-1/2} x) - \tr(v^{\ast \sharp}))$

\If{$\lambda^{\ast} \leq (1+C\epsilon\log^2(1/\epsilon))Q_{G'}(p^{\ast})$}
\State \textbf{return} $G'$ \label{returnGStep}
\EndIf
\State  Let $\mu$ be the median value of $p^{\ast}(X)$ over $X\in S'$.
\State Find a $T \geq C'$ so that
$$
\Pr_{X\in_u S'} (|p^{\ast}(X)-\mu| \geq T + 4/3) \geq \tail(T, d, \eps, \tau)
$$
\label{thresholdStep}
\State \textbf{return} $S'' = \{X\in S': |p^{\ast}(X)-\mu| < T\}.$ \label{filterStep}

\EndProcedure
\end{algorithmic}
\caption{Filter algorithm for a Gaussian with unknown covariance matrix.}
\label{alg:filter-Gaussian-covariance}
\end{algorithm}

In \cite{DKKLMS}, we take $\tail(T, d, \eps, \tau)= 12 \exp(-T) + 3\epsilon/(d \log(N/\tau))^2$, 
where $N=\Theta((d\log(d/\eps\tau))^6/\eps^2)$ is the number of samples we took there. 

To get a near sample-optimal algorithms, we will need a weaker definition of a good set. 
To use this,  we will need to weaken the tail bound in the algorithm to $\tail(T, d, \eps, \tau)= \eps/(T^2 \log^2(T))$, 
when $ T \geq 10\log(1/\eps)$. For $T \leq 10\log(1/\eps)$, we take $\tail(T, d, \eps, \tau)=1$ 
so that we always choose $ T \geq 10\log(1/\eps)$. It is easy to show that the integrals of this tail bound 
used in the proofs of Lemma 8.19 and Claim 8.22 of \cite{DKKLMS} have similar bounds. 
Thus, our analysis here will sketch that these tail bounds hold for a set of $\Omega(d^2 \log^5(d/\eps\tau)/\eps^2)$ samples from the Guassian.

Firstly, we state the new, weaker, definition of a good set:

\begin{definition} \label{def:good-sample-cov}
Let $G$ be a Gaussian in $\R^d$ with mean $0$ and covariance $\Sigma$. 
Let $\epsilon>0$ be sufficiently small. We say that a multiset $S$ of points in $\R^d$ is $\eps$-good with respect to $G$ if the following hold:
\begin{enumerate}
\item \label{farPoints} For all $x\in S$, $x^T \Sigma^{-1} x < d +O(\sqrt{d}\log(d/\eps))$.
\item \label{covariance} We have that $\|\Sigma^{-1/2} \Cov(S) \Sigma^{-1/2} - I\|_F = O(\eps)$.
\item \label{cocovariance} For all even degree-$2$ polynomials $p$, we have that $\Var(p(S)) = \Var(p(G))(1+O(\eps))$.
\item \label{tails} For $p$ an even degree-$2$ polynomial with $\E[p(G)]=0$ and $\Var(p(G))=1$, and for any $T> 10\log(1/\eps)$ we have that
$$
\Pr_{x\in_u S}(|p(x)| > T) \leq \eps/(T^2 \log^2(T)).
$$
\end{enumerate}
\end{definition}

It is easy to see that the algorithm and analysis of ~\cite{DKKLMS} can be pushed through using the above weaker definition.
That is, if $S$ is a good set, then $G$ can be recovered to $\tilde O(\eps)$ error from an $\eps$-corrupted version of $S$. 
Our main task will be to show that random sets of the appropriate size are good with high probability.

\begin{proposition} \label{lem:random-good-gaussian-cov}
Let $N$ be a sufficiently large constant multiple of $d^2 \log^5(d/\eps)/\eps^2$. 
Then a set $S$ of $N$ independent samples from $G$ is $\eps$-good with respect to $G$ with high probability.
\end{proposition}
\begin{proof}
First, note that it suffices to prove this when $G=N(0,I)$.

Condition \ref{farPoints} follows by standard concentration bounds on $\|x\|_2^2$.

Condition \ref{covariance} follows by estimating the entry-wise error between $\Cov(S)$ and $I$.

Condition \ref{cocovariance} is slightly more involved. 
Let $\{p_i\}$ be an orthonormal basis for the set of even, degree-$2$, mean-$0$ polynomials with respect to $G$. 
Define the matrix $M_{i,j} = \E_{x\in_u S}[p_i(x)p_j(x)]-\delta_{i,j}$. This condition is equivalent to $\|M\|_2 = O(\eps)$. 
Thus, it suffices to show that for every $v$ with $\|v\|_2 = 1$ that $v^TMv = O(\eps)$. 
It actually suffices to consider a cover of such $v$'s. Note that this cover will be of size $2^{O(d^2)}$. 
For each $v$, let $p_v = \sum_i v_i p_i$. We need to show that $\Var(p_v(S)) = 1 +O(\eps)$. 
We can show this happens with probability $1-2^{-\Omega(d^2)}$, and thus it holds for all $v$ in our cover by a union bound.

Condition \ref{tails} is substantially the most difficult of these conditions to prove. 
Naively, we would want to find a cover of all possible $p$ and all possible $T$, 
and bound the probability that the desired condition fails. Unfortunately, the best a priori bound on $\Pr(|p(G)| > T)$ 
are on the order of $\exp(-T)$. As our cover would need to be of size $2^{d^2}$ or so, 
to make this work with $T=d$, we would require on the order of $d^3$ samples in order to make this argument work.

However, we will note that this argument is sufficient to cover the case of $T<10\log(1/\eps)\log^2(d/\eps)$.

Fortunately, most such polynomials $p$ satisfy much better tail bounds. 
Note that any even, mean zero polynomial $p$ can be written in the form $p(x) = x^T A x - \tr(A)$ for some matrix $A$. 
We call $A$ the associated matrix to $p$. We note by the Hanson-Wright inequality that $\Pr(|p(G)| > T) = \exp(-\Omega(\min((T/\|A\|_F)^2,T/\|A\|_2))).$ 
Therefore, the tail bounds above are only as bad as described when $A$ has a single large eigenvalue. 
To take advantage of this, we will need to break $p$ into parts based on the size of its eigenvalues. We begin with a definition:

\begin{definition}
Let $\mathcal{P}_{k}$ be the set of even, mean-$0$, degree-$2$ polynomials, so that the associated matrix $A$ satisfies:
\begin{enumerate}
\item $\rank(A)\leq k$
\item $\|A\|_2 \leq 1/\sqrt{k}$.
\end{enumerate}
\end{definition}
Note that for $p\in \mathcal{P}_k$ that $|p(x)| \leq |x|^2/\sqrt{k} + \sqrt{k}$.

Importantly, any polynomial can be written in terms of these sets.
\begin{lemma}
Let $p$ be an even, degree-$2$ polynomial with $\E[p(G)] = 0, \Var(p(G))=1$. 
Then if $t=\lfloor \log_2(d) \rfloor$, it is possible to write $p=2(p_1+p_2+\ldots+p_{2^t}+p_d)$ where $p_k\in \mathcal{P}_k$.
\end{lemma}
\begin{proof}
Let $A$ be the associated matrix to $p$. 
Note that $\|A\|_F = \Var{p} = 1$. Let $A_{k}$ be the matrix corresponding to the top $k$ eigenvalues of $A$. 
We now let $p_1$ be the polynomial associated to $A_1/2$, $p_2$ be associated to $(A_2-A_1)/2$, 
$p_4$ be associated to $(A_4-A_2)/2$, and so on. It is clear that $p=2(p_1+p_2+\ldots+p_{2^t}+p_d)$. 
It is also clear that the matrix associated to $p_k$ has rank at most $k$. 
If the matrix associated to $p_k$ had an eigenvalue more than $1/\sqrt{k}$, 
it would need to be the case that the $k/2^{nd}$ largest eigenvalue of $A$ had size at least $2/\sqrt{k}$. 
This is impossible since the sum of the squares of the eigenvalues of $A$ is at most $1$.

This completes our proof.
\end{proof}

We will also need covers of each of these sets $\mathcal{P}_k$.
\begin{lemma}
For each $k$, there exists a set $\mathcal{C}_k\subset \mathcal{P}_k$ so that
\begin{enumerate}
\item For each $p\in \mathcal{P}_k$ there exists a $q\in \mathcal{C}_k$ so that $\|p(G)-q(G)\|_2 \leq (\eps/d)^2$.
\item $|\mathcal{C}_k| = 2^{O(dk\log(d/\eps))}.$
\end{enumerate}
\end{lemma}
\begin{proof}
We note that any such $p$ is associated to a matrix $A$ of the form $A = \sum_{i=1}^k \lambda_i v_i v_i^T$, 
for $\lambda_i \in [0,1/\sqrt{k}]$ and $v_i$ orthonormal. It suffices to let $q$ correspond to the matrix 
$A' = \sum_{i=1}^k \mu_i w_i w_i^T$ for with $|\lambda_i -\mu_i| < (\eps/d)^3$ and $|v_i-w_i| < (\eps/d)^3$ for all $i$. 
It is easy to let $\mu_i$ and $w_i$ range over covers of the interval and the sphere with appropriate errors. 
This gives a set of possible $q$'s of size $2^{O(dk\log(d/\eps))}$ as desired. 
Unfortunately, some of these $q$ will not be in $\mathcal{P}_k$ as they will have eigenvalues that are too large. 
However, this is easily fixed by replacing each such $q$ by the closest element of $\mathcal{P}_k$. 
This completes our proof.
\end{proof}

We next will show that these covers are sufficient to express any polynomial.
\begin{lemma}
Let $p$ be an even degree-$2$ polynomial with $\E[p(G)]=0$ and $\Var(p(G))=1$. 
It is possible to write $p$ as a sum of $O(\log(d))$ elements of some $\mathcal{C}_k$ plus another polynomial of $L^2$ norm at most $\eps/d$.
\end{lemma}
\begin{proof}
Combining the above two lemmas we have that any such $p$ can be written as
$$
p = (q_1 + p_1) + (q_2 + p_2) + \ldots (q_{2^t}+p_{2^t}) + (q_d+p_d) = q_1+q_2+\ldots+q^{2^t}+q^d + p' \;,
$$
where $q_k$ above is in $\mathcal{C}_k$ and $\|p_k(G)\|_2 < (\eps/d)^2$. 
Thus, $p'=p_1+p_2+\ldots+p_{2^t}+p_d$ has $\|p'(G)\|_2 \leq (\eps/d)$. 
This completes the proof.
\end{proof}
The key observation now is that if $|p(x)| \geq T$ for $\|x\|_2 \leq \sqrt{d/\eps}$, 
then writing $p=q_1+q_2+q_4+\ldots+q_d+p'$ as above, it must be the case that 
$|q_k(x)| > (T-1)/(2\log(d))$ for some $k$. Therefore, to prove our main result, 
it suffices to show that, with high probability over the choice of $S$, 
for any $T\geq 10\log(1/\eps)\log^2(d/\eps)$ and any $q\in \mathcal{C}_k$ 
for some $k$, 
that $\Pr_{x\in_u S}(|q(x)| > T/(2\log(d))) < \eps / (2 T^2 \log^2(T) \log(d))$. 
Equivalently, it suffices to show that for $T\geq 10 \log(1/\eps)\log(d/\eps)$ it holds 
$\Pr_{x\in_u S}(|q(x)| > T/(2\log(d))) < \eps / (2 T^2 \log^2(T) \log^2(d))$. 
Note that this holds automatically for $T>(d/\eps)$, as $p(x)$ cannot possibly be that large 
for $\|x\|_2 \leq \sqrt{d/\eps}$. Furthermore, note that losing a constant factor in the probability, 
it suffices to show this only for $T$ a power of $2$.

Therefore, it suffices to show for every $k\leq d$, every $q\in \mathcal{C}_k$ 
and every $d/\sqrt{k\eps} \gg T \gg \log(1/\eps)\log(d/\eps)$ that 
with probability at least $1-2^{-\Omega(dk\log(d/\eps))}$ over the choice of $S$ we have that 
$\Pr_{x\in_u S}(|q(x)|> T) \ll \eps/(T^2 \log^4(d/\eps))$. However, by the Hanson-Wright inequality, 
we have that 
$$\Pr(|q(G)| > T) = \exp(-\Omega(\min(T^2,T\sqrt{k}))) < (\eps/(T^2 \log^4(d/\eps)))^2 \;.$$ 
Therefore, by Chernoff bounds, the probability that more than a $\eps/(T^2 \log^4(d/\eps))$-fraction of the elements of $S$ satisfy this property is at most
\begin{align*}
\exp(-\Omega(\min(T^2,T\sqrt{k}))|S|\eps/(T^2 \log^4(d/\eps))) & = \exp(-\Omega(|S|\eps/(\log^4(d/\eps))\min(1,\sqrt{k}/T)))\\
& \leq \exp(-\Omega(|S|\eps^2/(\log^4(d/\eps))k/d))\\
& \leq \exp(-\Omega(dk\log(d/\eps))) \;,
\end{align*}
as desired.

This completes our proof.
\end{proof}

\section{Omitted Details from Section 5}
\subsection{Full description of the distributions for experiments}
\label{sec:exp-setup}
Here we formally describe the distributions we used in our experiments.
In all settings, our goal was to find noise distributions so that noise points were not ``obvious'' outliers, in the sense that there is no obvious pointwise pruning process which could throw away the noise points, which still gave the algorithms we tested the most difficulty.
We again remark that while other algorithms had varying performances depending on the noise distribution, it seemed that the performance of ours was more or less unaffected by it.

\paragraph{Distribution for the synthetic mean experiment}
Our uncorrupted points were generated by $\normal (\mu, I)$, where $\mu$ is the all-ones vector.
Our noise distribution is given as
\[
N = \frac{1}{2} \Pi_1 + \frac{1}{2} \Pi_2 \; ,
\]
where $\Pi_1$ is the product distribution over the hypercube where every coordinate is $0$ or $1$ with probability $1/2$, and $\Pi_2$ is a product distribution where the first coordinate is ether $0$ or $12$ with equal probability, the second coordinate is $-2$ or $0$ with equal probability, and all remaining coordinates are zero.

\paragraph{Distribution for the synthetic covariance experiment}
For the isotropic synthetic covariance experiment, our uncorrupted points were generated by $\normal (0, I)$, and the noise points were all zeros. 
For the skewed synthetic covariance experiment, our uncorrupted points were generated by $\normal (0, I + 100 e_1 e_1^T)$, where $e_1$ is the first unit vector, and our noise points were generated as follows: we took a fixed random rotation of points of the form $Y_i \sim \Pi$, where $\Pi$ is a product distribution whose first $d/2$ coordinates are each uniformly selected from $\{-0.5, 0, 0.5\}$, and whose next $d / 2 - 1$  coordinates are each $0.8 \times A_i$, where for each coordinate $i$, $A_i$ is an independent random integer between $-2$ and $2$, and whose last coordinate is a uniformly random integer between $[-100, 100]$.

\paragraph{Setup for the semi-synthetic geographic experiment}
We took the 20 dimensional data from \cite{novembre2008genes}, which was diagonalized, and randomly rotated it.
This was to simulate the higher dimensional case, since the singular vectors that \cite{novembre2008genes} obtained did not seem to be sparse or analytically sparse.
Our noise was distributed as $\Pi$, where $\Pi$ is a product distribution whose first $d/2$ coordinates are each uniformly random integers between $0$ and $2$ and whose last $d / 2$ coordinates are each uniformly randomly either $2$ or $3$, all scaled by a factor of $1/24$.

\subsection{Comparison with other robust PCA methods on semi-synthetic data}
\label{sec:exp-others}
In addition to comparing our results with simple pruning techniques, as we did in Figure 3 in the main text, we also compared our algorithm with implementations of other robust PCA techniques from the literature with accessible implementations.
In particular, we compared our technique with RANSAC-based techniques, \texttt{LRVCov}, two SDPs (\cite{CLMW11, xu2010robust}) for variants of robust PCA, and an algorithm proposed by \cite{CLMW11} to speed up their SDP based on alternating descent.
For the SDPs, since black box methods were too slow to run on the full data set (as \cite{CLMW11} mentions, black-box solvers for the SDPs are impractical above perhaps 100 data points), we subsample the data, and run the SDP on the subsampled data.
For each of these methods, we ran the algorithm on the true data points plus noise, where the noise was generated as described above.
We then take the estimate of the covariance it outputs, and project the data points onto the top two singular values of this matrix, and plot the results in Figure \ref{fig:europe2}.

Similar results occurred for most noise patterns we tried.
We found that only our algorithm and \texttt{LRVCov} were able to reasonably reconstruct Europe, in the presence of this noise.
It is hard to judge qualitatively which of the two maps generated is preferable, but it seems that ours stretches the picture somewhat less than \texttt{LRVCov}.

\begin{figure*}
\begin{tikzpicture}
\node[inner sep=0pt] (pruning) at (.2\textwidth, 0\textheight)
    {\includegraphics[width=.44\textwidth]{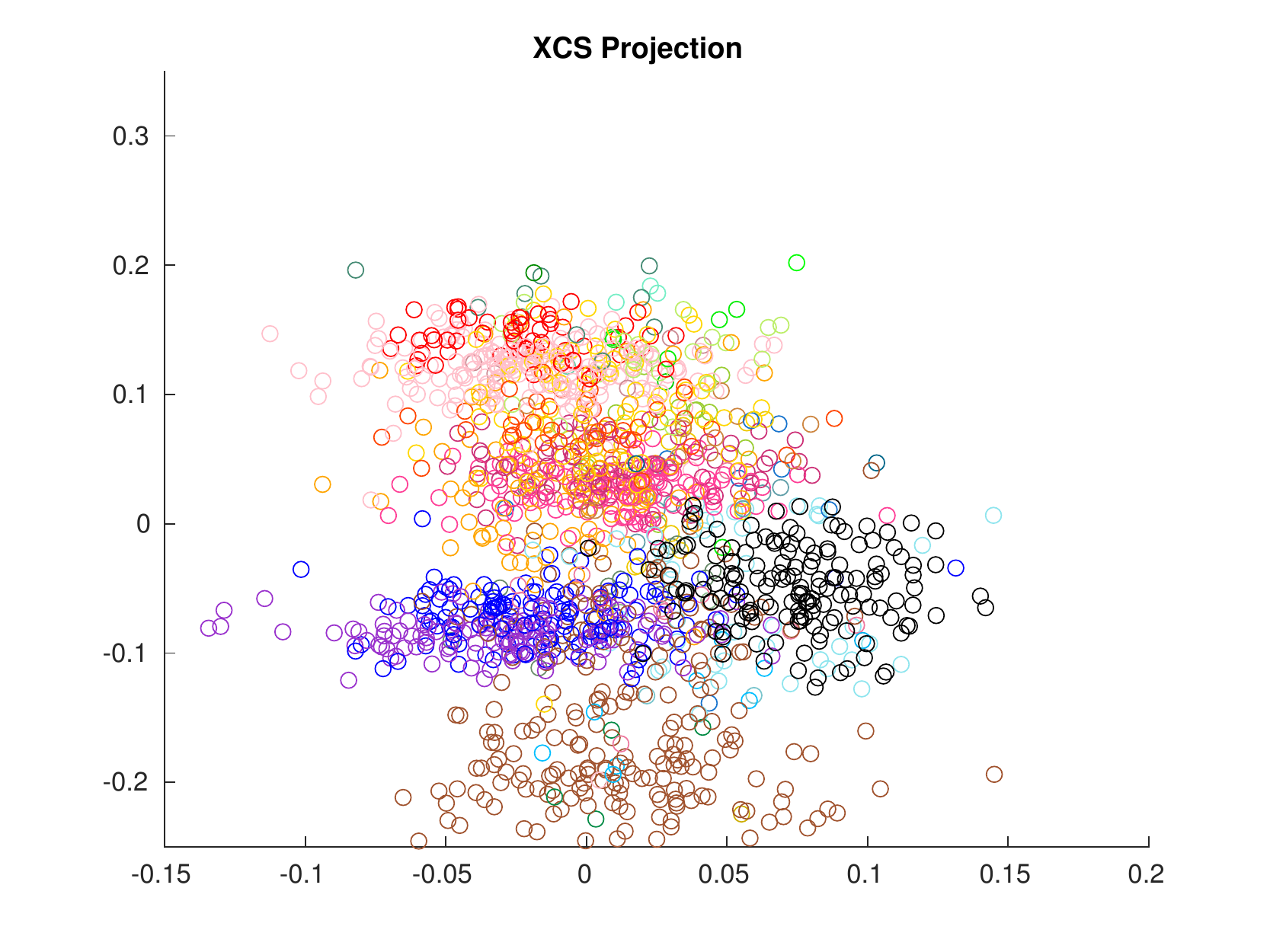}};
\node[inner sep=0pt] (pruning) at (.75\textwidth, .23 \textheight)
    {\includegraphics[width=.44\textwidth]{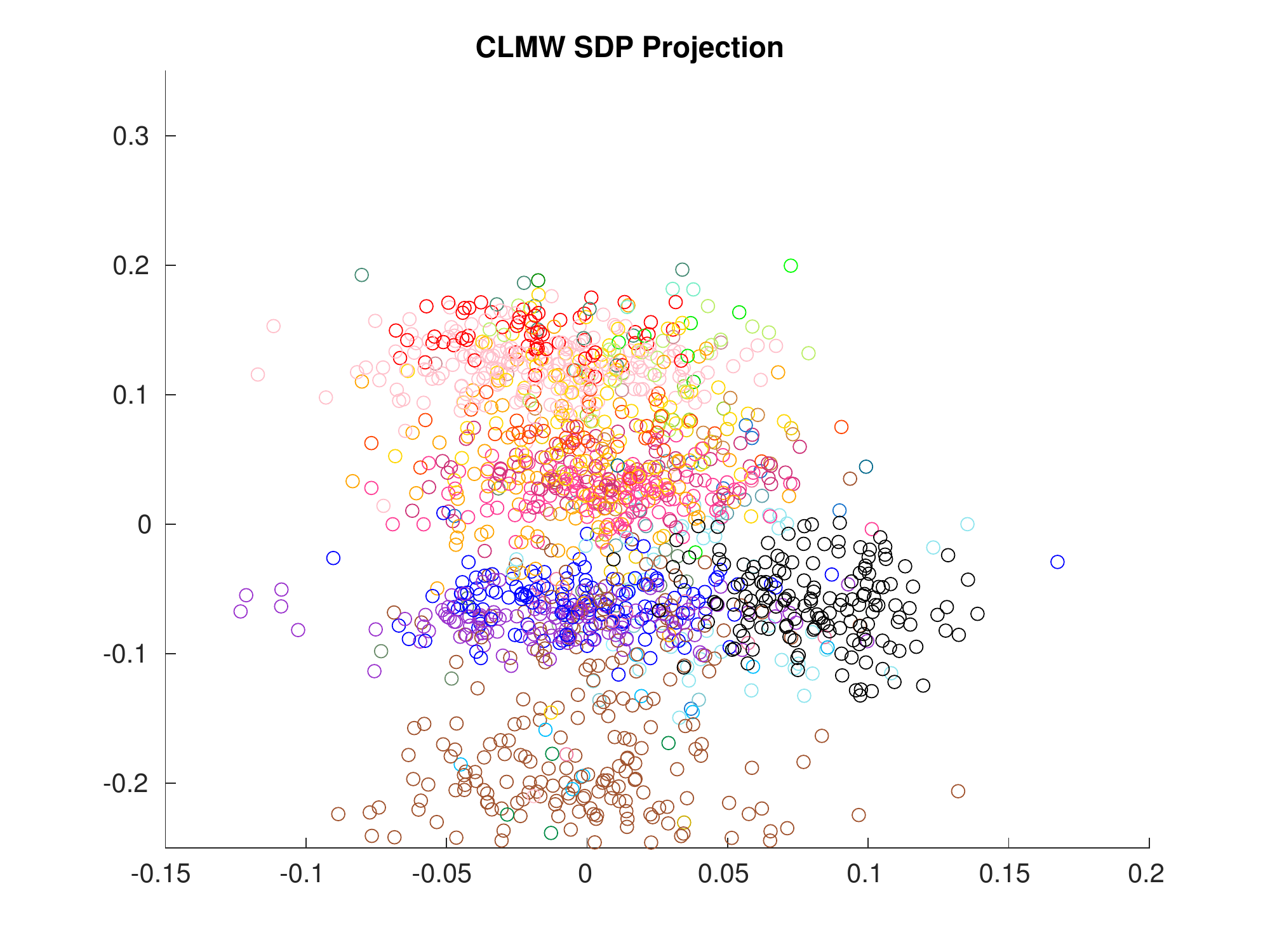}};
\node[inner sep=0pt] (original) at (.2\textwidth, .23\textheight)
    {\includegraphics[width=.44\textwidth]{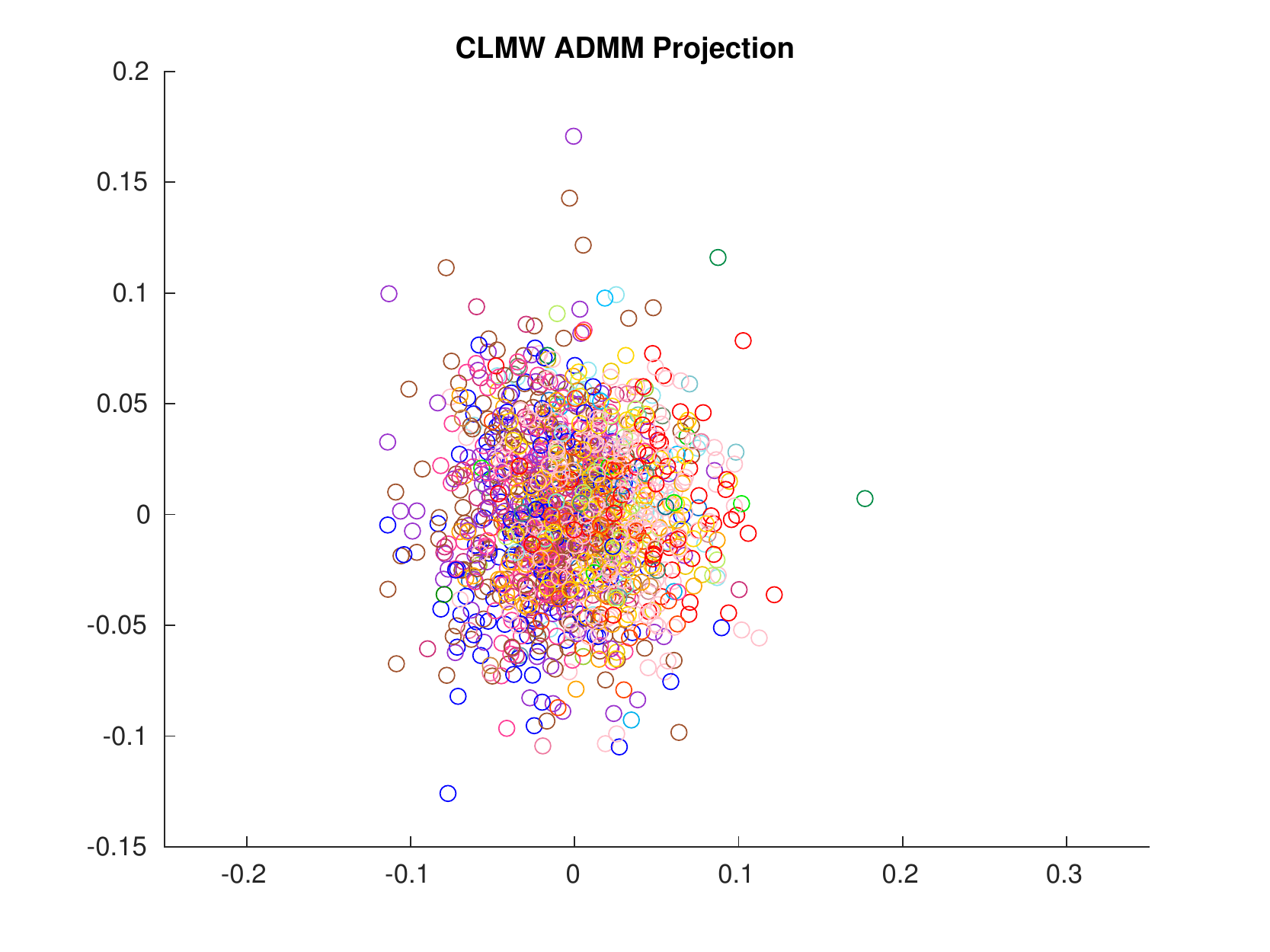}};
\node[inner sep=0pt] (pruning) at (.2\textwidth, .46\textheight)
    {\includegraphics[width=.44\textwidth]{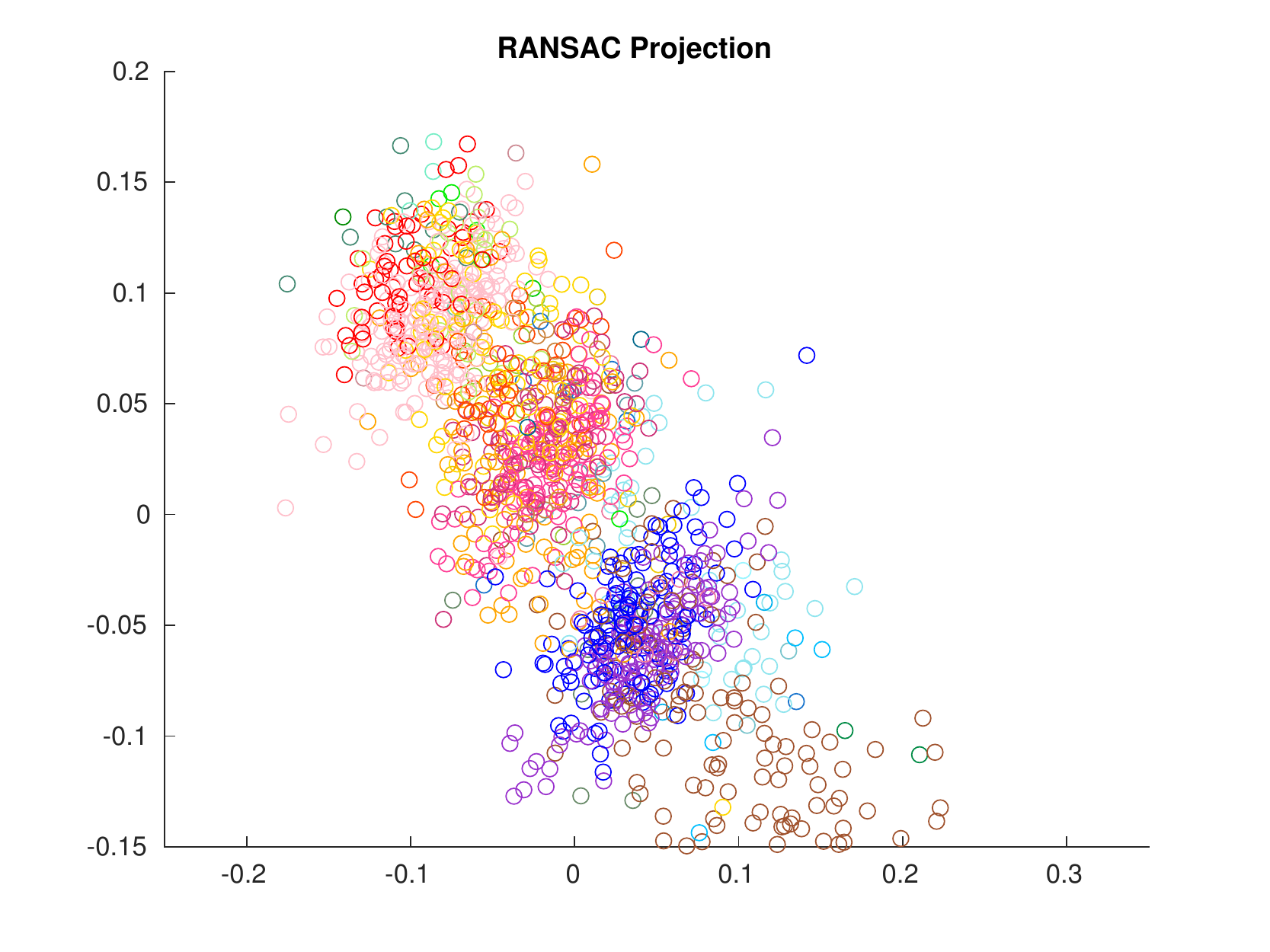}};
\node[inner sep=0pt] (pruning) at (.75 \textwidth, .46\textheight)
    {\includegraphics[width=.44\textwidth]{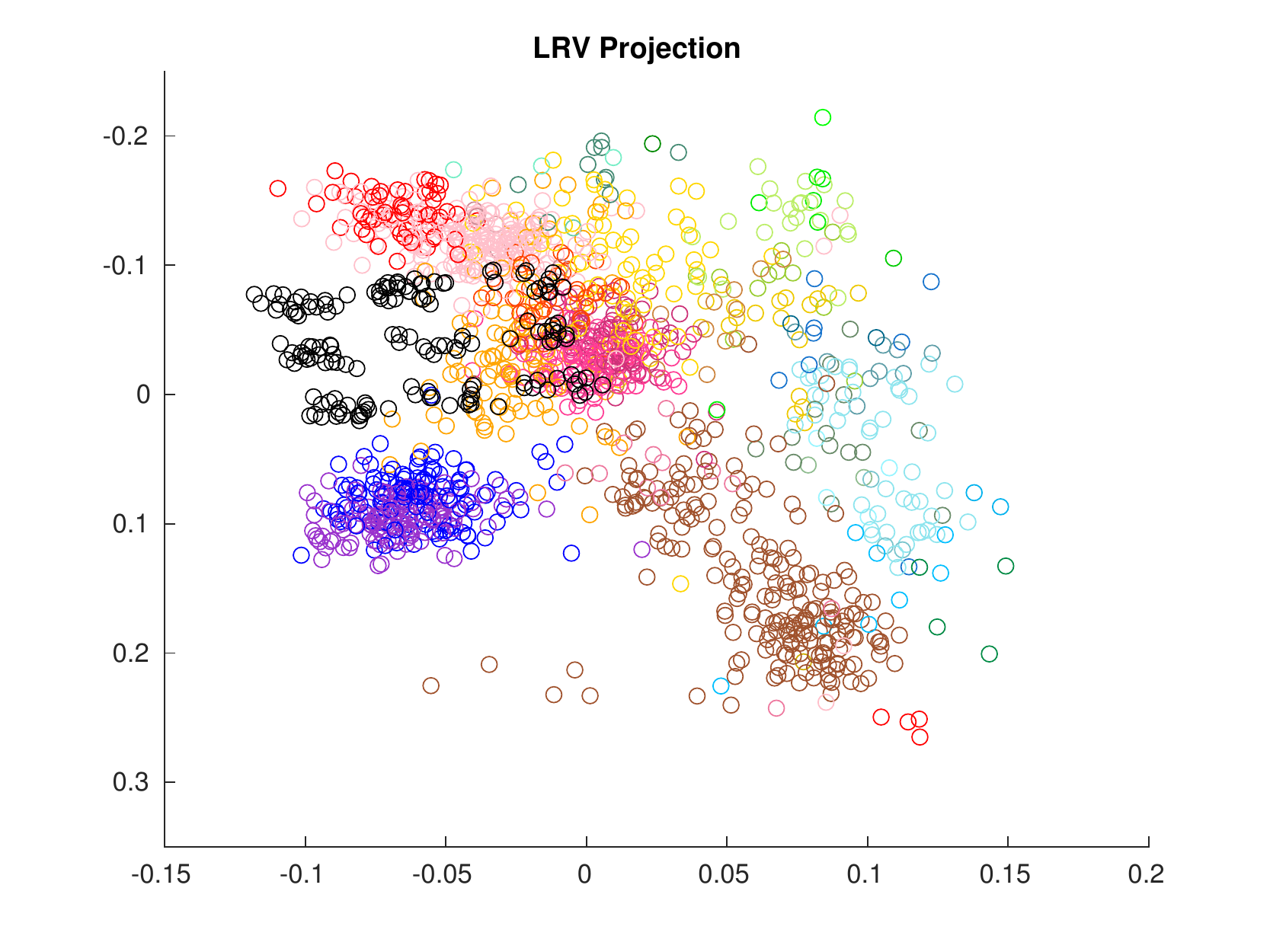}};
\node[inner sep=0pt] (pruning) at (.2 \textwidth, .69\textheight)
    {\includegraphics[width=.44\textwidth]{plot_data/original-eps-converted-to.pdf}};
\node[inner sep=0pt] (pruning) at (.75\textwidth, .69\textheight)
    {\includegraphics[width=.44\textwidth]{plot_data/filterproj-eps-converted-to.pdf}};
\end{tikzpicture}
\caption{Comparison with other robust methods on the Europe semi-synthetic data. From left to right, top to bottom: the original projection without noise, what our algorithm recovers, RANSAC, \texttt{LRVCov}, the ADMM method proposed by \cite{CLMW11}, the SDP proposed by \cite{xu2010robust} with subsampling, and the SDP proposed by \cite{CLMW11} with subsampling.}
\label{fig:europe2}
\end{figure*}

\end{document}